\documentclass[10pt]{article} 
\usepackage{geometry}
\usepackage{amsmath}
\usepackage{amssymb, amsfonts}
\usepackage{mathtools}        
\usepackage{bm}
\usepackage{mathrsfs}
\mathtoolsset{showonlyrefs=true}
\usepackage[round, sort&compress]{natbib}
\makeatletter
\let\proof\@undefined
\let\endproof\@undefined
\makeatother
\usepackage{amsthm}
\usepackage{cite}
\usepackage[dvipsnames]{xcolor}
\usepackage{graphicx}
\usepackage{booktabs,array,tabularx,multirow}
\usepackage{algorithm}
\usepackage{algpseudocode}
\usepackage{microtype}

\usepackage{caption}
\makeatletter
\let\labelindent\@undefined
\makeatother
\usepackage{enumitem}
\setlist{nosep}

\usepackage{hyperref}
\hypersetup{
    colorlinks=true,
    linkcolor=magenta,
    filecolor=magenta,
    urlcolor=Cerulean,
    citecolor=Cerulean,
    pdftitle={LIMBO: Learning and Internalizing Model-Free Barrier Objectives for Agile and Safe Whole-Body Control}
}

\let\cite\citep
\theoremstyle{definition}
\newtheorem{theorem}{Theorem}
\newtheorem{proposition}{Proposition}
\newtheorem{lemma}{Lemma}
\newtheorem{corollary}{Corollary}
\theoremstyle{definition}
\newtheorem{assumption}{Assumption}

\newtheorem{remark}{Remark}
\usepackage{xcolor}

\definecolor{bluefire}{HTML}{08AAE3}
\definecolor{vitaminc}{HTML}{FF9900}
\definecolor{AmazonGold}{HTML}{cea968}

\DeclareMathOperator*{\argmin}{arg\,min}
\DeclareMathOperator*{\argmax}{arg\,max}

\DeclareMathOperator{\KL}{KL}

\DeclareMathOperator*{\E}{\mathbb{E}}
\newcommand{\R}{\mathbb{R}}

\newcommand{\cX}{\mathcal{X}}
\newcommand{\cU}{\mathcal{U}}

\newcommand{\cB}{\mathcal{B}}
\newcommand{\cC}{\mathcal{C}}
\newcommand{\cD}{\mathcal{D}}

\newcommand{\cK}{\mathcal{K}}
\newcommand{\cL}{\mathcal{L}}

\newcommand{\cR}{\mathcal{R}}

\newcommand{\cXsafe}{\mathcal{X}_s}
\newcommand{\cXfail}{\mathcal{X}_f}
\newcommand{\pitype}[1]{\pi^{\mathrm#1}}
\newcommand{\pitask}{\pitype{t}}
\newcommand{\pibase}{\pitype{b}}
\newcommand{\pisafe}{\pitype{s}}

\newcommand{\marginfn}{g}

\newcommand{\classkl}{\alpha}
\newcommand{\entropyreg}{\kappa}
\newcommand{\rtask}{r}
\newcommand{\rteach}{\bar{r}}
\newcommand{\cLsafe}{\mathcal{L}_{\mathrm{s}}}

\newcommand{\Ir}{\mathrm{I}}
\newcommand{\dimresid}{n_r}
\newcommand{\scinot}[1]{\times 10^{-#1}}
\newif\ifnotes\notestrue

\title{
LIMBO: Learning and Internalizing Model-Free Barrier Objectives for Agile and Safe  Whole-Body Control
}

\author{
Jake Gonzales$^{1,2,\dagger}$,
Arturo Flores Alvarez$^{1,3,\dagger}$,
Yu-Ming Chen$^{1}$,\\
Aaron D. Ames$^{1,4}$,
Lillian J. Ratliff$^{1,2,*}$,
and Manikantan Nambi$^{1,*}$
}

\date{}

\begin{document}

\maketitle

\begingroup
\centering
$^{1}$Amazon;
$^{2}$University of Washington, Seattle;\\
$^{3}$University of California, Los Angeles;
$^{4}$California Institute of Technology
\par
\endgroup

\begingroup
\renewcommand{\thefootnote}{\fnsymbol{footnote}}
\footnotetext[2]{Work done while interning at Amazon.}
\footnotetext[1]{Co-lead.}
\endgroup

\pagestyle{plain}

\begin{abstract}
Safe whole-body control requires coordinating collision avoidance and balance under high-dimensional, nonlinear dynamics---making safety certificates difficult to design and reuse across behaviors. We present LIMBO, a framework for  synthesizing a state--action control barrier function and distilling its safety structure into a task policy. LIMBO learns the safety certificate from black-box transitions and a state-based failure specification over residual actions around a frozen base controller, making Q-CBF synthesis tractable in the full control dimension while placing the certificate in the task policy’s control space. During synthesis, the learned safety value drives risk-guided sampling near the estimated boundary of recoverability; during task learning, it serves as a teacher that provides action-level safety feedback, yielding a robust task policy and alleviating the need for an online safety filter at deployment. We demonstrate LIMBO on a 29-degree-of-freedom humanoid performing dodgeball avoidance and locomotion beneath low obstacles. Beyond scaling learned Q-CBFs to whole-body control, we show that risk-guided boundary sampling provides a theoretically grounded way to explore the edge of recoverability. Under the same safety specification, \emph{ceteris paribus}, varying the sampling concentration produces strategies ranging from crouching to a novel backward-leaning limbo maneuver. In both settings, the learned policies transfer to hardware without online safety filtering, showing that learned safety synthesis  scales to agile whole-body control. Website: \url{https://limbo-safe-wbc.com/}. 
\end{abstract}

\section{Introduction}
\label{sec:intro}

High-dimensional robotic systems such as humanoids are increasingly capable of agile and expressive whole-body behaviors \citep{yang2026pacmanperceptionawarecbfrlwholebody, cheng2024RSS, ji2024exbody2,he2024iros}. Safety at this scale is inherently coupled across many degrees of freedom, contact, balance, and nonlinear dynamics.
A motion that avoids one collision may destabilize the robot, create contact elsewhere on the body, or leave the system in an unrecoverable state. Safe whole-body control therefore requires reasoning about coordinated motion across the full system rather than any subsystem in isolation.

Control barrier functions (CBFs) provide a principled framework for safety-critical control by restricting actions to preserve a forward-invariant safe set~\citep{ames2019cbf}. A central challenge, however, is barrier function synthesis. Standard approaches typically rely on analytical barriers designed for particular constraints and operating conditions, together with sufficient knowledge of the system dynamics to enforce them. For complex robotic systems, this can require substantial modeling effort or reduced-order safety representations that omit dynamics important for agile motion. The resulting certificates may depend on expert system knowledge, privileged geometric information, and carefully chosen safety representations, making them difficult to scale or transfer.

\begin{figure}[t]
    \centering
    \includegraphics[width=0.5\columnwidth]{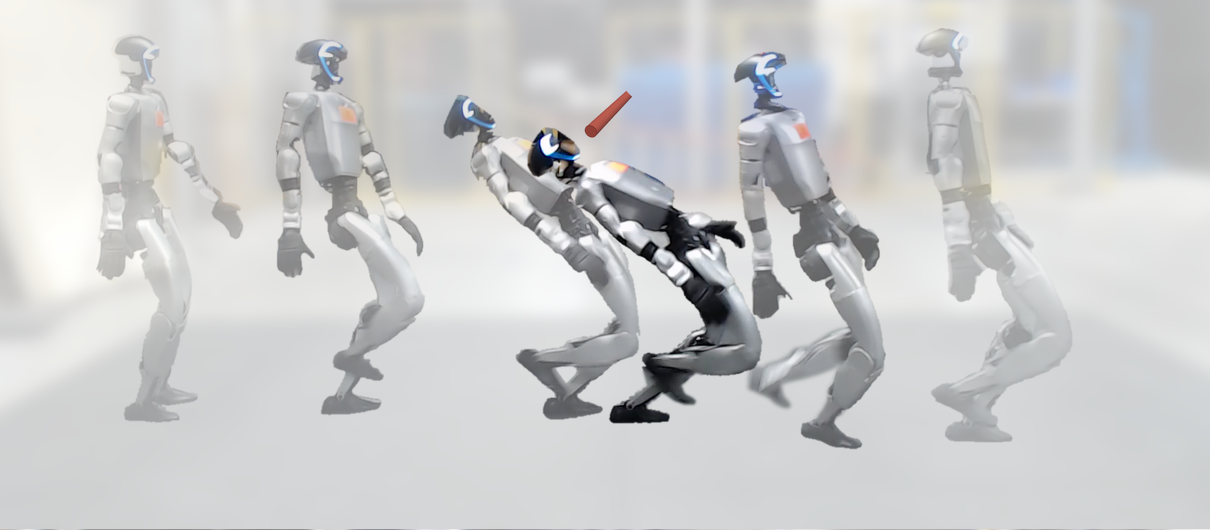}\includegraphics[width=0.5\columnwidth]{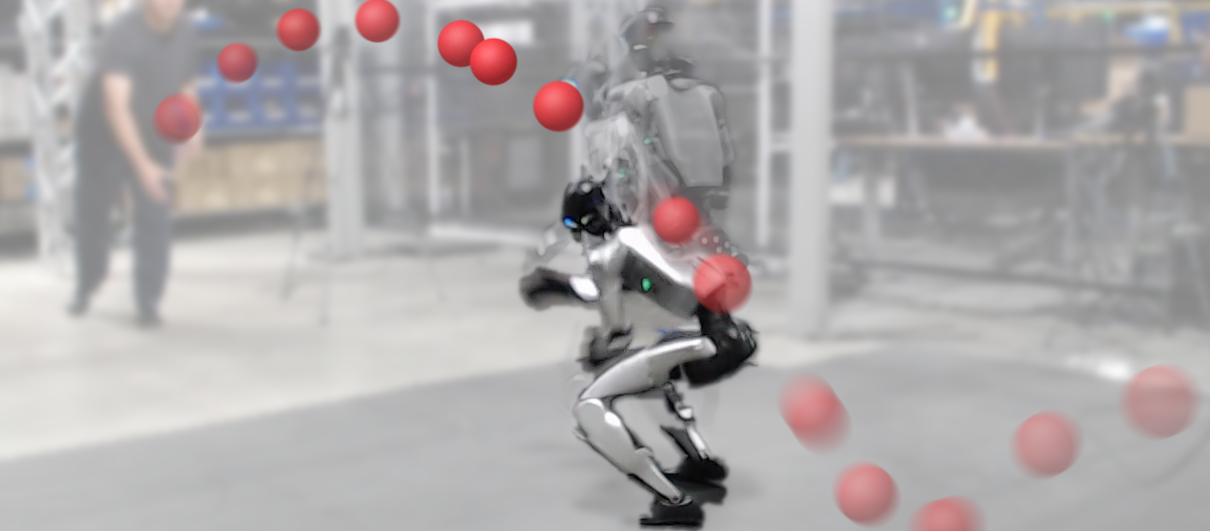}
    \caption{\textbf{Safety synthesized and distilled into agile behavior.} The deployed $29$-DoF humanoid performs backward-leaning limbo  (left) and dodgeball evasion (right) using policies trained from an automatically synthesized Q-CBF and subsequent safe policy distillation.}
    \label{fig:teaser}
\end{figure}

These limitations have motivated a broad class of learning-based approaches to safety-critical control, including methods that learn safety certificates, models, constraints, or safe policies directly from data~\citep{dawson2023safe}. Two directions are particularly relevant. The first seeks to \emph{learn the safety certificate itself}: Q-CBFs lift the barrier into state--action space and learn a safety value function that captures the future consequences of candidate actions from black-box system transitions~\citep{oh2025safety,oh2026robustqcbf}. 
Given a state-based safety margin that defines failure, the learned Q-CBF identifies actions from which safety can continue to be maintained. A second direction seeks to \emph{internalize safety into the task policy}. CBF-RL~\citep{yang2026cbfrl}, e.g., incorporates an analytical CBF into policy training via safety-filtered rollouts and CBF-guided objectives (albeit hand-crafted), allowing the learned policy to satisfy the prescribed safety constraints without requiring the filter at deployment. 

These approaches address complementary parts of the problem: Q-CBFs learn the safety certificate yet have primarily been used as runtime safety filters, whereas CBF-RL transfers safety into the policy yet relies on an analytical barrier specified \emph{a priori}. What is missing is a framework that can scale learned safety certificates to high-dimensional control and then distill that learned safety structure into a task policy.
LIMBO addresses this gap. 

\subsection{Contributions}
The contributions are summarized as follows:
\begin{itemize}[itemsep=5pt, topsep=5pt]
\item \textbf{Scalable Q-CBF synthesis \& policy distillation.}
Central to LIMBO is a residual Q-CBF formulation that places the learned certificate in the same control space as the task policy while retaining a frozen base controller for nominal stabilization or locomotion. This makes Q-CBF synthesis tractable over the full control dimension using only black-box transitions and a state-based failure specification. The learned safety value then serves as a training-time teacher, providing action-level Q-CBF violations and counterfactual corrections to guide the task policy to internalize safety without replacing its objective. This enables deployment without an online safety filter.
\item \textbf{Provable safety guarantees and transfer to the task policy.}
We characterize the exact safety value through the viability kernel (Proposition~\ref{prop:viability-value}) and quantify how teacher margins accommodate approximation and distillation errors (Proposition~\ref{prop:qcbf-teacher-transfer} and Corollary~\ref{cor:learned-qcbf-teacher-transfer}). Theorem~\ref{thm:main-stage-one-safety} gives an explicit lower bound on the probability that the Stage~I safety policy keeps the system in the certified safe set throughout a finite horizon, starting from a state in that set. Theorem~\ref{thm:main-stage-two-safety} gives the corresponding bound for the Stage~II task policy deployed without an online safety filter. These bounds account for errors in learned safety values and in action selection or policy distillation.
\item \textbf{Risk-guided discovery at the edge of recoverability.}
We formulate boundary-focused replay as a risk-seeking change of measure that controls how strongly synthesis explores the edge of recoverability. We prove exponential concentration on states with near-maximal boundary utility as the tilting strength increases (Theorem~\ref{thm:boundary-tilting}). Under the same safety specification, stronger boundary concentration shifts the learned strategy from crouching to a novel backward-leaning limbo maneuver, without prescribing either behavior. To our knowledge, this is the first demonstration of humanoid limbo behavior emerging from safety synthesis and deployed on hardware.

\item \textbf{Agile whole-body sim-to-real control.}
LIMBO is  validated on a 29-degree-of-freedom humanoid performing dodgeball avoidance and locomotion beneath low obstacles (aka limbo). Both Q-CBF synthesis and
task-policy training are performed entirely in simulation, yet the resulting policies transfer to hardware without additional adaptation or online safety filtering, demonstrating that learned safety synthesis scales to robust whole-body control.

\end{itemize}

\subsection{Related Work }
\label{sec:related-work}
Safety-critical control provides principled tools for reasoning about failure and invariance, including viability theory, reachability analysis, and CBFs \citep{aubin:inria-00636570,BansalHJReachability2017,ames2019cbf}. Their practical use, however, depends on constructing a safety representation that is both tractable and expressive.
Reduced-order models and geometric constraints are common tools for making safety computation tractable in whole-body robotics, but they can exclude coupled effects that are important for agile motion. 
Learned certificates offer a different route by approximating Lyapunov functions, barriers, or safety values  from data \citep{dawson2023safe,so2024train}. 
Q-CBFs are especially relevant as they
lift safety into state--action space, so actions can be
evaluated via black-box transitions rather than an analytical dynamics model
\citep{oh2025safety,oh2026robustqcbf}.

A complementary line of work uses model-based or control-theoretic structure to guide policy learning rather than relying on it only at deployment~\citep{yang2026cbfrl, li2025clfrlcontrollyapunovfunction}. This is attractive for high-dimensional control: structured guidance can shape the learning problem while leaving a lightweight policy to execute the final behavior. CBF-RL applies this idea directly to safety, using an analytical CBF to filter training rollouts and provide corrective reward feedback that is progressively internalized by the policy \citep{yang2026cbfrl}. PAC-MAN brings this approach to humanoid dodgeball, demonstrating that barrier guidance can produce agile collision-avoidance behavior and transfer to hardware \citep{yang2026pacmanperceptionawarecbfrlwholebody}. While promising, these results assume a pre-specified barrier. Our framework closes this gap: the safety certificate is first synthesized from interaction and is then reused to supervise task policy training.

Recent learned whole-body controllers demonstrate that humanoids can
execute increasingly dynamic and expressive behaviors  \citep{cheng2024RSS,omnih2o2024,asap2025, beyondmimic2025}. Many such approaches rely on reference motion or imitation to organize the otherwise difficult whole-body learning problem. Adversarial motion priors (AMP), e.g., regularize a policy toward a distribution of plausible motion without requiring it to track a prescribed trajectory \citep{amp2021}. We use this separation deliberately: the motion prior regularizes \emph{how} the robot moves, while the learned Q-CBF determines which actions are compatible with safety. Consequently, safe behavior is not encoded by a reference trajectory or hand-designed barrier, allowing distinct whole-body responses to emerge under the same safety specification.

\begin{figure*}[t]
    \centering
    \includegraphics[width=1.0\textwidth]{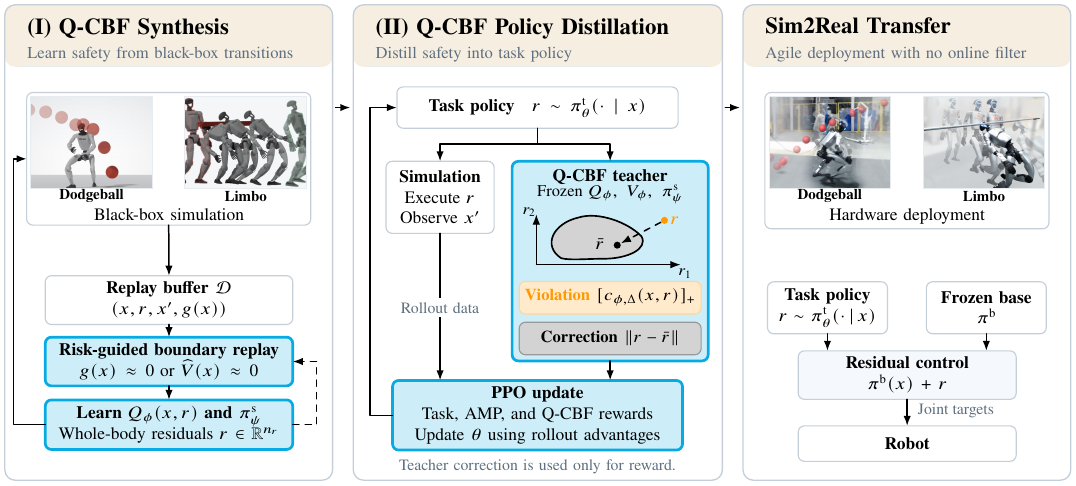}
    \caption{\textbf{LIMBO Overview.}~\textbf{\emph{Stage~I}}: Q-CBF synthesis via black-box  transitions using the residual formulation and risk-guided boundary replay. \textbf{\emph{Stage~II}}: learned Q-CBF is a \emph{safety teacher}, shaping task-policy rewards via constraint violations and corrections.  \textbf{\emph{Deployment}}: compose task policy with the frozen base controller, with optional filtering. Hardware demonstrations operate without online filtering, relying on the safety behavior internalized  during Stage~II.}
    \label{fig:training}
\end{figure*}

\section{Problem Formulation}
\label{sec:problem}
Consider a deterministic discrete-time system: 
\begin{equation}
\label{eq:system}
    x_{t+1}=f(x_t,u_t),
    \quad
    x_t\in\cX\subseteq\R^{n_x},
    \quad
    u_t\in\cU\subseteq\R^{n_u},
\end{equation}
where $\cU$ denotes the set of admissible control inputs. We do not assume that the dynamics $f$ are available analytically. Instead, given a state--action pair $(x,u)$ a simulator or the physical system provides the successor state $f(x,u)$. In particular, our formulation does not require derivatives of the dynamics or a control-affine model. 

Throughout, let $\boldsymbol{u}=(u_0,\ldots,u_{H-1})$ denote a finite-horizon control sequence of length $H\geq0$, and let $\boldsymbol{x}=(x_0,\ldots,x_H)$ denote the corresponding state trajectory, with $x_{k+1}=f(x_k,u_k)$ for $k=0,\ldots,H-1$. We assume that the system is initialized at a known state $x_0=x$ and that the admissible control set $\cU$ is compact.
Let  $\mathrm{USC}(\cX)$ be the space of upper semicontinuous functions, and $\mathrm{USC}_b(\cX)$ be those that are additionally bounded.
\paragraph{Safety Specification.}
Safety is specified through a bounded margin function $\marginfn:\cX \to \R$, whose sign determines whether a state satisfies the safety specification.
Let 
$\cXsafe := \{x\in\cX : \marginfn(x)\geq 0\}\subset \cX$ 
be the constraint-satisfying set, and let $\cXfail:=\cX\setminus\cXsafe=\{x\in\cX:\marginfn(x)<0\}$ be the failure set. Thus, the function $\marginfn(x)$ represents a signed safety margin rather than a pre-specified barrier certificate. Multiple failure conditions can be combined via $\marginfn(x) = \min_{j\in[m]} \marginfn_j(x)$, where $[m] := \{1, \ldots, m\}$ indexes the $m$ component safety conditions.
The goal is to characterize, from black-box transitions, which actions preserve the safety specification $\marginfn$ under the system dynamics. We construct the corresponding state-space viability value and then lift it to state--action space to obtain a Q-CBF.

\subsection{Maximal Viability Value}
We construct the infinite-horizon safety value via a sequence of
finite-horizon problems. This establishes existence of the infinite-horizon value and identifies the solution of the undiscounted Bellman equation relevant
to viability.

\begin{assumption}[Regularity]
\label{ass:viability-regularity}
The admissible control set $\cU\subseteq\R^{n_u}$ is nonempty and
compact, the dynamics $f\in\mathcal C(\cX\times\cU;\cX)$, and
the margin $\marginfn\in\mathcal C(\cX;\R)$ is bounded.
\end{assumption}
For $H\geq0$, define the finite-horizon safety value
\[V_H(x)
    :=
    \max_{\boldsymbol{u}\in\cU^H}
    \min_{0\leq k\leq H}\marginfn(x_k),\]
where $x_0=x$, $x_{k+1}=f(x_k,u_k)$, and set $V_0(x):=\marginfn(x)$.
Define the Bellman operator
\begin{equation}
\label{eq:safety-bellman-operator}
    (\mathcal{T}V)(x)
    :=\textstyle
    \min\big\{
        \marginfn(x),
        \max_{u\in\cU}V\bigl(f(x,u)\bigr)\big\}.
\end{equation}

 Under Assumption~\ref{ass:viability-regularity}, the maximum is attained whenever $V\in \mathrm{USC}(\cX)$, and  the finite-horizon values satisfy
$V_0=\marginfn$ and $V_{H+1}=\mathcal{T}V_H$.
The following proposition shows that this recursion admits a well-defined infinite-horizon limit.

\begin{proposition}[Infinite-horizon viability value]
\label{prop:viability-value}
Under Assumption~\ref{ass:viability-regularity}, for every $x\in\cX$,
\[V^\star(x):=\lim_{H\to\infty}V_H(x)=\inf_{H\geq0}V_H(x)\in\mathbb{R}.\]
Moreover, the optimal value function $V^\star$ is upper semi-continuous and bounded (i.e., $V^\star\in\mathrm{USC}_b(\cX)$),  satisfies
$\mathcal TV^\star=V^\star$ and, pointwise,
\[V^\star=
\sup\bigl\{
W\in\mathrm{USC}_b(\cX):
\mathcal TW=W
\bigr\}.\] 
Finally,  
$\cC^\star(\cXsafe) := \{x\in\cX:V^\star(x)\geq0\}$ 
is the maximal controlled-invariant subset of $\cXsafe$.
\end{proposition}
The full proof is given in Appendix~\ref{app:viability-proof}.
Extending the safety horizon produces a decreasing sequence of
finite-horizon values, with a finite limit because the safety margin is
bounded. The key step in the proof is to use compactness and continuity
to pass the Bellman maximization to this limit. Monotonicity identifies
$V^\star$ as the greatest bounded upper semicontinuous fixed point, and
attained Bellman maximizers preserve its zero-superlevel set. Finally,
every controlled-invariant safe set lies in this zero-superlevel set,
establishing maximality.

Moreover, according to Proposition~\ref{prop:viability-value}, the set $\cC^\star(\cXsafe)$ is the viability kernel of $\cXsafe$. Thus $V^\star(x)\geq0$ if and only if there exists an admissible control sequence that maintains safety indefinitely from $x$, while $V^\star(x)<0$ indicates that failure is unavoidable under every admissible control strategy.

\subsection{Q-CBF Formulation} 
\label{sec:qcbf}
To characterize the safety of individual actions, we lift the viability value to state--action space and define
\begin{equation}
\label{eq:qcbf-q}
    Q^\star(x,u)
    :=
    \min\left\{
        \marginfn(x),
        V^\star\bigl(f(x,u)\bigr)
    \right\}.
\end{equation}
Thus $\mathcal{T}V^\star = V^\star$ and \eqref{eq:qcbf-q} imply that $Q^\star$ satisfies
\begin{equation}
\label{eq:q-to-v}
    V^\star(x)
    =
    \max_{u\in\cU} Q^\star(x,u).
\end{equation}
Hence, the optimal state-action value $Q^\star(x,u)$ measures the worst safety margin achievable over the infinite future when action $u$ is applied at the current state and control is subsequently chosen optimally for safety.

Let $\classkl:[0,\infty)\to[0,\infty)$ be a class-$\cK$ function\footnote{A function $\classkl:[0,\infty)\to[0,\infty)$ is class-$\cK$ if it is continuous, strictly increasing, and satisfies $\classkl(0)=0$.} satisfying $\classkl(r)\leq r$ for all $r\geq0$. For
$x\in\cC^\star(\cXsafe)$, consider actions $u\in\cU$ satisfying the Q-CBF constraint
\begin{equation}
        Q^\star(x,u)
        \geq
        \classkl\bigl(V^\star(x)\bigr).
\end{equation}
Since $x\in\cC^\star(\cXsafe)$ implies
$V^\star(x)\in[0,\marginfn(x)]$ and
$\classkl(V^\star(x))\leq V^\star(x)$, this constraint is equivalent to
$V^\star\bigl(f(x,u)\bigr)
    \geq
    \classkl\bigl(V^\star(x)\bigr)$.
The Q-CBF constraint is feasible for every
$x\in\cC^\star(\cXsafe)$. Indeed, since the maximum in
\eqref{eq:q-to-v} is attained, any
$u^\star(x)\in
    \argmax_{u\in\cU}Q^\star(x,u)$
satisfies
\[Q^\star(x,u^\star(x))
    =
    V^\star(x)
    \geq
    \classkl(V^\star(x)).\]
Furthermore, if $x_t\in\cC^\star(\cXsafe)$ and
$u_t\in\cU$ satisfies the Q-CBF constraint, then
\[V^\star(x_{t+1})
    \geq
    \classkl\bigl(V^\star(x_t)\bigr)
    \geq 0,\]
and hence $x_{t+1}\in\cC^\star(\cXsafe)$. This is to say that the viability kernel
$\cC^\star(\cXsafe)$ is forward invariant under any control sequence satisfying
the Q-CBF constraint. Therefore $V^\star$ provides an upper semicontinuous discrete-time barrier certificate whose zero-superlevel set is precisely the maximal controlled-invariant safe set; continuity or differentiability of $V^\star$ is not needed for this invariance argument. The Q-CBF constraint can be stricter than $Q^\star(x,u)\geq0$: when $V^\star(x)>0$, it also requires retaining the positive margin $\classkl(V^\star(x))$.

\subsection{Scaling via Residual Control}
To support safety learning in high-dimensional systems, we parameterize
control through a residual $r_t\in\mathcal R\subseteq\R^{\dimresid}$ around
a frozen base policy $\pibase$ that provides nominal stabilization or
locomotion. This supplies a control prior: $r=0$ recovers the nominal
behavior, and appropriately scaled residual exploration can focus on
modifications needed to satisfy the safety specification. The base
policy remains active throughout the resulting motion, so the Q-CBF
evaluates residual actions together with the base controller's feedback
response. Using the same residual interface for synthesis and downstream
task control enables task policy supervision (cf.~Section~\ref{sec:safety-distillation}).

Let $\Gamma(x,r)$ compose the base-policy output and residual into an
admissible system input, with $0\in\mathcal R$ and $\Gamma(x,0)$ recovering the base.
Absorbing this composition into the dynamics, we have
\begin{equation}
\label{eq:whole-body-dynamics}
    x_{t+1}=f_{\mathrm{res}}(x_t,r_t),
    \quad
    f_{\mathrm{res}}(x,r):=f\bigl(x,\Gamma(x,r)\bigr).
\end{equation}
Holding $\pibase$ and the command interface fixed keeps these effective dynamics unchanged across synthesis, policy training, and evaluation. Under Assumption~\ref{ass:viability-regularity} for $f_{\mathrm{res}}$
and $\mathcal R$, the preceding construction yields
$Q^\star:\cX\times\mathcal R\to\R$. From this point onward,
$V^\star$, $Q^\star$, and $\cC^\star(\cXsafe)$ denote the safety
values and viability kernel of this residual-control system. The
maximal safe set is therefore relative to the fixed base policy,
command interface, and admissible residual controls.

For example, in humanoid whole-body control, we use $\dimresid=29$ and
$q_t^{\mathrm{des}}=\Gamma(x_t,r_t)
    =q^{\mathrm{def}}+S\bigl(\pibase(x_t)+r_t\bigr)$,
where $q^{\mathrm{def}}$ is the nominal joint configuration and
$S\in\R^{\dimresid\times \dimresid}$ is a fixed diagonal scaling matrix.
Our experiments demonstrate Q-CBF synthesis and policy distillation
over this full residual action space, capturing coordinated effects
across the robot's limbs, contacts, and floating base.

\section{Learning High-Dimensional  Q-CBFs}
\label{sec:qcbf-learning}
The formulation in Section~\ref{sec:qcbf} characterizes the ideal infinite-horizon safety value $Q^\star$, which is generally intractable for high-dimensional dynamics. We therefore approximate it directly from black-box transitions of the effective system in \eqref{eq:whole-body-dynamics}. In particular, we learn a state--action safety critic $Q_\phi:\cX\times \mathcal{R} \to\R$ together with a safety policy $\pisafe_\psi:\cX\to\mathcal{R}$ that approximately maximizes the learned safety value. By a slight abuse of notation, we use $\pisafe_\psi(\cdot\,|\,x)$ to denote the distribution of actions under $\pisafe_\psi$ conditioned on state $x$, and $\pisafe_\psi(x)$ to denote a sample from that distribution. 

\subsection{Discounted Q-CBF Synthesis}
Directly learning the undiscounted recursion in~\eqref{eq:safety-bellman-operator} may be difficult in practice. We instead use a discounted safety Bellman recursion with discount factor $\gamma\in [0,1)$, and increase $\gamma$ toward one during training. For a transition $(x_t,r_t,x_{t+1})$, we construct the target 
\begin{equation}
\label{eq:qcbf-target}
    y_t
    =
    (1-\gamma)\marginfn_t
    +
    \gamma
    \min\left\{
        \marginfn_t,
        Q_{\bar{\phi}}
        \bigl(
            x_{t+1},
            \pisafe_\psi(x_{t+1})
        \bigr)
    \right\},
\end{equation}
where $\marginfn_t=\marginfn(x_t)$ and $Q_{\bar{\phi}}$ denotes a target critic. For terminal transitions, we set $y_t=\marginfn(x_t)$. The corresponding learned state safety value is $V_\phi(x) := \mathbb{E}_{r\sim\pisafe_\psi(\cdot\,|\,x)} Q_\phi(x,r)$. Unlike a conventional reward Bellman backup, the  target recursion \eqref{eq:qcbf-target} contains only the safety margin $\marginfn$. In particular, using the current-state margin inside the backup preserves the ordering $Q(x,u)\leq\marginfn(x)$ at the Bellman fixed point so that states with $\marginfn(x) < 0$ cannot belong to the zero super-level set of the resulting safety value. As $\gamma \to 1$, the discounted recursion approaches the undiscounted recursion \eqref{eq:safety-bellman-operator} with $V^\star$---namely,
\[    (\mathcal{T}V^\star)(x)
    =
    \min\left\{
        \marginfn(x),
        \max_{u\in\cU}V^\star\bigl(f(x,u)\bigr)\right\}.\] 
        This connects the learned critic to $\cC^\star(\cXsafe)$. 
We thus use 
$\cC_\phi
    :=
    \{x\in\cX:V_\phi(x)\geq0\}$
as the learned approximation of the maximal controlled-invariant set $\cC^\star(\cXsafe)$. 

\subsection{Learning Procedure}
\label{sec:qcbf-learning-procedure}
We learn $Q_\phi$ from a replay buffer $\cD$ of black-box transitions using the Bellman target in \eqref{eq:qcbf-target}. 
We maintain an ensemble of $n$ critics and define the learned safety value as their pointwise minimum, $Q_\phi(x,r):=\min_{1\leq j\leq n}Q_{\phi_j}(x,r)$. Each critic is trained against the same target by minimizing
\begin{equation}
\label{eq:critic-loss}
\cL_Q(\phi)
=
\E_{(x,r,x')\sim\cD}
\left[
    \frac{1}{n}\sum_{j=1}^{n}
    \bigl(Q_{\phi_j}(x,r)-y\bigr)^2
\right].
\end{equation}
Expectations over $\cD$ follow the replay sampling rule described below. The target parameters are updated by Polyak averaging to stabilize the bootstrapped targets, and $Q_{\bar\phi}$ denotes the pointwise minimum of the corresponding target critics. 

The safety policy $\pisafe_\psi$ is trained to approximate the maximizing action in $V^\star(x)
    =
    \max_{u\in\cU} Q^\star(x,u)$. Its core objective combines the learned safety value with entropy regularization,
\begin{equation}
\label{eq:safety-policy-loss}
\cLsafe(\psi)
=
-\E_{x\sim\cD} \left[\E_{r\sim\pisafe_\psi(\cdot\mid x)}
\bigl[
    Q_\phi(x,r)
    -
    \entropyreg\log\pisafe_\psi(r\mid x)
\bigr]\right],
\end{equation}
where $\entropyreg\geq0$ controls entropy regularization. The entropy term encourages exploration during synthesis, while the Q-value objective favors actions with high predicted recoverability. Entropy regularization enters only the policy objective and leaves the safety Bellman target unchanged. Together, the critic and safety policy approximate the state--action safety value and maximizing control action  in the viability recursion.

\subsection{Risk-guided boundary exploration}
Learning the Q-CBF requires accurate value estimates near the boundary between recoverable and unrecoverable states, i.e., near $V^\star(x)=0$. Nominal rollouts tend to concentrate well inside the safe region, leaving this boundary sparsely represented in replay. Value-guided exploration concentrates learning on these informative states to resolve the safety boundary more precisely. Coupled with critic and safety-policy updates, this sampling can reveal viable behaviors whose safety is not obvious a priori, such as backward-leaning limbo, without prescribing an avoidance strategy or changing the safety specification.

We formulate boundary sampling as a regularized risk-seeking change of measure. Let $\mu$ denote a reference probability distribution over states in the replay buffer and define the boundary utility $b_\phi(x) := -|V_\phi(x)|$. For $\beta>0$, we choose the replay distribution by maximizing over probability measures $\rho$ absolutely continuous with respect to $\mu$:
\begin{equation}
\label{eq:boundary-tilt}
\rho_\beta
\in\textstyle
\argmax_{\rho\ll\mu}
\left\{
\mathbb{E}_{X\sim\rho}[b_\phi(X)]
-
\tfrac{1}{\beta}
\KL(\rho|\mu)
\right\}.
\end{equation}
The following result characterizes the unique optimizer and the tradeoff between reference replay and concentration on high boundary utility.

\begin{theorem}[KL-optimal replay and boundary concentration]
\label{thm:boundary-tilting}
Fix $\phi$, let $\mu$ be a probability measure, and let $b_\phi$ be
bounded and measurable. For any $\beta\geq0$, define
$Z_\beta:=\mathbb E_{X\sim\mu}[e^{\beta b_\phi(X)}]$ and 
$
\frac{d\rho_\beta}{d\mu}(x):=
e^{\beta b_\phi(x)}/Z_\beta$.
The following statements hold:
\begin{enumerate}[label={\it\alph*.},leftmargin=*,itemsep=3pt,topsep=3pt]
\item For every $\beta>0$, the distribution $\rho_\beta$ is the unique optimizer of
\eqref{eq:boundary-tilt}, and the optimal objective value is
$\beta^{-1}\log Z_\beta$.
\item Let $b_\phi^\star:=\operatorname*{ess\,sup}_{X\sim\mu}b_\phi(X)$.
Replay concentrates
exponentially on states with near-maximal boundary utility as
$\beta\to\infty$: for every $\beta\geq0$ and $0<\delta<\varepsilon$, the estimate holds
\[
\rho_\beta(b_\phi\leq b_\phi^\star-\varepsilon)
\leq
e^{-\beta(\varepsilon-\delta)}/
\mu(b_\phi\geq b_\phi^\star-\delta).
\]
The denominator is strictly positive. 
\item For every finite $\beta\geq0$, the distributions $\rho_\beta$ and $\mu$ are
mutually absolutely continuous and have the same support.
Moreover, $\rho_0\equiv \mu$ and $\rho_\beta\to\mu$ in total variation
as $\beta\to 0$.
\end{enumerate}
\end{theorem}
The full proof is provided in Appendix~\ref{app:risk_proofs}.
For $b_\phi(x)=-|V_\phi(x)|$, the optimal replay weights are
proportional to $\exp(-\beta|V_\phi(x)|)$. This is the upper-entropic,
or risk-seeking, tilt of the boundary utility: increasing $\beta$
favors available replay states with smaller absolute learned safety
values, while $\beta\downarrow0$ recovers the reference distribution.

The replay construction accommodates different approximations
of the safety value; e.g., given candidate residuals
$r_j\sim\pisafe_\psi(\cdot\,|\, x)$, a target-critic estimate is
$\widehat V_{\bar\phi}(x)
:=\max_{1\leq j\leq K}Q_{\bar\phi}(x,r_{j})$.
Using $b(x)=-|\widehat V_{\bar\phi}(x)|$ focuses replay
near the estimated recoverability boundary. The zeroth-order
approximation $V_0=\marginfn$ instead gives
$b_0(x)=-|\marginfn(x)|$, emphasizing the instantaneous
constraint boundary $\marginfn(x)=0$. This sampler cannot distinguish
states with the same current margin but different future recoverability.
Value-guided replay uses the learned safety value to make that distinction.
Section~\ref{sec:limbo} and Figure~\ref{fig:boundary-replay} show how it
shifts replay toward the estimated viability boundary and the
margin-defined constraint boundary, regions that uniform replay scarcely
samples. Concentrating learning near loss of recoverability reveals
the backward-leaning limbo strategy
(Figure~\ref{fig:beta-sweep}). Margin-guided replay can also discover limbo,
but its most aggressive setting in Figure~\ref{fig:boundary-replay}
collapses, whereas strong value-guided replay supports stable limbo.

\begin{algorithm}[t]
\small
\caption{LIMBO: Synthesis \& Policy Distillation}
\label{alg:qcbf-distillation}
\begin{algorithmic}[1]
\Require  $\pibase$, $\marginfn$,
 $\Delta\geq0$, replay rule $\mathsf{Sample}$
\State Initialize $Q_\phi$, $\pisafe_\psi$, $\bar\phi\gets\phi$,
and replay buffer $\cD$

\Statex \textbf{Stage I: Q-CBF synthesis}
\For{$k=1,\ldots,K$}
    \State Collect transitions under $\pisafe_\psi(\cdot\mid x)$ into $\cD$
    \State $\cB\gets\mathsf{Sample}(\cD)$ and $r'\gets\pisafe_\psi(x')$
    \State $y\gets(1-\gamma_k)\marginfn(x)
    +\gamma_k\min\{\marginfn(x),Q_{\bar\phi}(x',r')\}$
    \State Update $\phi,\psi$ using \eqref{eq:critic-loss}
    and \eqref{eq:safety-policy-loss}; set $\bar\phi\gets(1-\tau)\bar\phi+\tau\phi$
\EndFor
\State Freeze $Q_\phi,\pisafe_\psi$; set
$V_\phi(x):=Q_\phi(x,\pisafe_\psi(x))$

\Statex \textbf{Stage II: Q-CBF policy distillation (QCBF-PD)}
\State Initialize $\pitask_\theta$ and AMP discriminator
\For{each policy update}
    \State Collect  rollout via $\pitask_\theta$, and 
 teacher corrections $\rteach$
    (Sec.~\ref{sec:qcbf-teacher})
    \State Form rewards $J_t$ using \eqref{eq:qcbf-guided-policy-objective}
    \State Update $\theta$ with PPO using rollout advantages
    \State Train discriminator on rollout and reference data
\EndFor
\State \Return $\pitask_\theta$
\end{algorithmic}
\end{algorithm}

\section{Q-CBF Guided Policy Training}
\label{sec:safety-distillation} 

A learned Q-CBF can be used as a runtime safety filter, but the task policy may still propose unsafe actions and depend on the filter for correction. Instead, LIMBO uses the learned Q-CBF as a \emph{training-time safety teacher}, transferring its safe-action structure into the task policy.
The teacher evaluates the policy's proposal and provides a corrective reference anchored to that proposal, guiding policy optimization toward the learned safety condition while preserving the task objective. This differs from imitating $\pisafe_\psi$, which is trained to maximize the learned safety value rather than accomplish the task.

\subsection{Q-CBF Teacher}
\label{sec:qcbf-teacher}
Let $\pitask_\theta(\cdot\,|\, x)$ be a task policy that acts via a residual around the base policy $\pibase$;  at state $x$,  $r\sim\pitask_\theta(\cdot\,|\, x)\in \mathcal R$  is the proposed admissible residual action. The safety critic $Q_\phi$ and  policy
$\pisafe_\psi$ remain fixed throughout this stage. 

Let $\alpha:\mathbb{R}\to\mathbb{R}$ be an extended class $\mathcal{K}$ function satisfying $\alpha(v)\leq v$ for $v\geq 0$.  
Fix $\Delta\geq0$ and define the learned Q-CBF constraint residual
$c_{\phi,\Delta}(x,r)
:=\alpha(V_\phi(x))+\Delta-Q_\phi(x,r)$, and admissible set $\mathcal{A}_{\Delta}(x):=\{r\in \mathcal{R}: c_{\phi,\Delta}(x,r)\leq0\}$. 
When $\mathcal{A}_{\Delta}(x)\neq \emptyset$, the ideal Q-CBF teacher selects a minimum-deviation correction
\begin{equation}
\label{eq:qcbf-training-filter}
\bar r \in \argmin_{\tilde r\in\mathcal{A}_{\Delta}(x)}\left\{ \frac{1}{2}\,\|\tilde r-r\|_W^2 \right\},
\end{equation}
where $W\succ0$ weights deviations in the residual action space. An admissible proposal is unchanged. A positive $\Delta$ requires additional Q-CBF slack that can accommodate approximation and distillation errors under the conditions of Section~\ref{sec:qcbf-safety-transfer}. Feasibility must be checked for the chosen margin.

For the untightened condition, the safety-policy action is feasible whenever $V_\phi(x)\geq0$. If no admissible correction is found, the safety-policy action is used only as a best-effort reference. In the following subsection, we use both the Q-CBF violation of $\rtask$ and the correction $\rteach-\rtask$ as supervisory signals to distill this safe-action structure into the task policy.

\subsection{Safety Distillation}
\label{sec:qcbf-distillation}
During Stage~II, the system uses the task-policy residual $r$ composed with the frozen base policy $\pibase$, not the teacher action $\bar r$. The correction is therefore \emph{counterfactual}: it provides supervision without intervening, so the policy learns from the consequences of its actions.

The task policy retains its nominal task reward, physical penalties, and adversarial motion prior. We augment these objectives with feedback
on both violation of the learned Q-CBF condition and the magnitude of the teacher correction. Specifically, we define the reward
\begin{equation}
\label{eq:qcbf-distillation-reward}
J_{\mathrm{Q},t}
=\max\left\{-C,\,-[c_{\phi,\Delta}(x_t,r_t)]_+
+\exp\left(-\|r_t-\bar r\|_W^2/\sigma_r^2\right)-1\right\},
\end{equation}
where $[z]_+:=\max\{z,0\}$. The  term $-[c_{\phi,\Delta}(x_t,r_t)]_+$ penalizes only violations by the unfiltered task proposal, whereas $\exp\bigl(-\|r_t-\bar r\|_W^2/\sigma_r^2\bigr)-1$ penalizes the amount of
correction indicated by the teacher, vanishing when $\bar r=r_t$ and approaching $-1$ for large deviations. The parameter $\sigma_r>0$ sets the action-deviation scale of this penalty; increasing it makes the penalty less sensitive to a given correction. The clipping constant $C>0$ bounds the combined reward to $[-C,0]$, limiting the influence of large violations. 

The two penalties are clipped jointly and share a single reward weight. Together, they provide state-dependent feedback on the task policy's proposed residual via the learned Q-CBF condition and the teacher correction. During training, this feedback is applied during active obstacle encounters, with $J_{\mathrm{Q},t}=0$ otherwise.

The PPO \citep{schulman2017proximal} per-step reward is
\begin{equation}
\label{eq:qcbf-guided-policy-objective}
J_t=J_{\mathrm{task},t}
+\delta_t\bigl(w_{\mathrm{Q}}J_{\mathrm{Q},t}
+\lambda_s J_{\mathrm{AMP},t}\bigr),
\end{equation}
where $J_{\mathrm{task},t}$ includes the task and physical-control rewards, $J_{\mathrm{AMP},t}$ encourages agreement with reference motions, and $\delta_t$ is the control period. The Q-CBF feedback enters the scalar reward before PPO advantage estimation, encouraging the policy to select actions that satisfy the learned safety condition in the current state.

\subsection{Transferring Safety to the Task Policy}
\label{sec:qcbf-safety-transfer}

The losses above encourage the task policy to reproduce Q-CBF-admissible actions, but successful distillation on sampled training states does not by itself establish safety of the unfiltered policy. We next quantify when proximity to a safe teacher action is sufficient to inherit the Q-CBF condition.

Define the exact Q-CBF admissible residual set \begin{equation} 
\label{eq:exact-qcbf-residual-set} \cR^\star_{\mathrm{QCBF}}(x) := \left\{ r\in\cR: Q^\star(x,r) \geq \classkl\bigl(V^\star(x)\bigr) \right\}.
\end{equation}

\begin{proposition}[Safety transfer] 
\label{prop:qcbf-teacher-transfer} 
Fix $x\in\cC^\star(\cXsafe)$ and let $Q^\star(x,\cdot)$ be $L_Q$-Lipschitz continuous on $\cR$ for some $L_Q>0$. Consider a teacher action $\rteach\in \mathcal{R}$ satisfying 
$Q^\star(x,\rteach) \geq \classkl\bigl(V^\star(x)\bigr) + \Delta$ for  $\Delta>0$. Then $\{r\in \mathcal{R}:\, \|\rtask-\rteach \| \leq \Delta/L_Q\}\subset \cR^\star_{\mathrm{QCBF}}(x)$.
\end{proposition} 
\begin{proof}
Take any $\rtask\in\cR$ satisfying
$\|\rtask-\rteach\|\leq\Delta/L_Q$.
Lipschitz continuity and the teacher margin give
\begin{align*}
Q^\star(x,\rtask)
&\geq Q^\star(x,\rteach)-L_Q\|\rtask-\rteach\|\\
&\geq \classkl\bigl(V^\star(x)\bigr)+\Delta-L_Q\|\rtask-\rteach\|\\
&\geq \classkl\bigl(V^\star(x)\bigr).
\end{align*}
Thus $\rtask\in\cR^\star_{\mathrm{QCBF}}(x)$.
\end{proof}

Proposition~\ref{prop:qcbf-teacher-transfer} gives the teacher slack a direct interpretation as a tolerance to distillation error. Actions lying deeper inside the Q-CBF admissible set can be imitated less precisely while still satisfying the safety condition.

The practical teacher is constructed from learned values, so we also account for approximation error.
\begin{corollary}[Safety transfer with a learned Q-CBF]
\label{cor:learned-qcbf-teacher-transfer}
Fix $x\in\cC^\star(\cXsafe)$, and assume $\classkl(\cdot)$ and $Q^\star(x,\cdot)$ are $L_\classkl$ and $L_Q$ Lipschitz continuous, respectively. Suppose $|V_\phi(x)-V^\star(x)|\leq\varepsilon_V$ and $|Q_\phi(x,r)-Q^\star(x,r)|\leq\varepsilon_Q$ uniformly in $r$. 
Let \eqref{eq:qcbf-training-filter} return $\rteach\in\mathcal{R}$, and consider a task-policy proposal $\rtask\in\mathcal{R}$ satisfying $\|\rtask-\rteach\|\leq\varepsilon_\pi$. Then
\begin{equation}
\label{eq:learned-qcbf-transfer-margin}
Q^\star(x,\rtask)-\classkl\bigl(V^\star(x)\bigr)
\geq
\Delta-\varepsilon_Q-L_\classkl\varepsilon_V-L_Q\varepsilon_\pi,
\end{equation}
so that 
$
\Delta
\geq
\varepsilon_Q+L_\classkl\varepsilon_V+L_Q\varepsilon_\pi$ implies $\rtask\in\cR^\star_{\mathrm{QCBF}}(x)$.
\end{corollary}
\begin{proof}
Feasibility of \eqref{eq:qcbf-training-filter}, the
approximation bounds, and Lipschitz continuity yield 
\begin{align*}
Q^\star(x,\rtask)
&\geq Q^\star(x,\rteach)-L_Q\varepsilon_\pi \\
&\geq Q_\phi(x,\rteach)-\varepsilon_Q-L_Q\varepsilon_\pi \\
&\geq \alpha(V_\phi(x))+\Delta
      -\varepsilon_Q-L_Q\varepsilon_\pi \\
&\geq \alpha(V^\star(x))+\Delta
      -\varepsilon_Q-L_\alpha\varepsilon_V-L_Q\varepsilon_\pi.
\end{align*}
Rearranging proves 
\eqref{eq:learned-qcbf-transfer-margin}.
The stated bound on $\Delta$ then implies
$\rtask\in\cR^\star_{\mathrm{QCBF}}(x)$.
\end{proof}

Corollary~\ref{cor:learned-qcbf-teacher-transfer} shows that the tightening
margin $\Delta$ provides a common budget for critic approximation error,
state-value approximation error, and imperfect distillation. In particular,
when $\varepsilon_Q=\varepsilon_V=0$, feasibility of the teacher problem gives
$Q^\star(x,\rteach)
\geq
\classkl\bigl(V^\star(x)\bigr)+\Delta$.
For any task-policy proposal satisfying
$\|\rtask-\rteach\|\leq\varepsilon_\pi$, Lipschitz continuity then implies
\[
Q^\star(x,\rtask)
\geq
Q^\star(x,\rteach)-L_Q\|\rtask-\rteach\|
\geq
\classkl\bigl(V^\star(x)\bigr)
\]
whenever $\Delta\geq L_Q\varepsilon_\pi$. Hence
$\rtask\in\cR^\star_{\mathrm{QCBF}}(x)$ as in Proposition~\ref{prop:qcbf-teacher-transfer}.

For $\Delta=0$, \eqref{eq:learned-qcbf-transfer-margin} still provides
an approximate Q-CBF bound, but it does not certify exact invariance
unless the combined approximation and distillation error vanishes.
Positive tightening is therefore what converts bounded learning and
distillation errors into an exact Q-CBF condition.

\begin{remark}[Feasibility of the tightened teacher]
\label{rem:learned-teacher-feasibility}
Corollary~\ref{cor:learned-qcbf-teacher-transfer} is conditional on
feasibility of the tightened teacher problem
\eqref{eq:qcbf-training-filter}. Tightening improves robustness to
learning and distillation error, but also requires sufficient Q-CBF
slack at the current state. In particular,
Proposition~\ref{prop:learned-teacher-feasibility} in
Appendix~\ref{app:hp-safety} shows that, under the stated approximation
bounds, the learned teacher is guaranteed to be feasible whenever
\begin{equation}
\label{eq:main-teacher-feasibility-condition}
V^\star(x)
-
\classkl\bigl(V^\star(x)\bigr)
\geq
\Delta
+
\varepsilon_Q
+
L_\classkl\varepsilon_V.
\end{equation}
That is, feasibility is guaranteed on a robust interior of the viability
kernel where the available exact Q-CBF slack is large enough to
accommodate both the fixed tightening and certificate approximation
error. 
\end{remark}

If the task policy selects Q-CBF-admissible residuals for every
$x\in\cC^\star(\cXsafe)$, then the invariance result of
Section~\ref{sec:qcbf} applies directly to the unfiltered task policy. Training encourages this condition on sampled states but does not establish it uniformly over $\cC^\star(\cXsafe)$. Exact infinite-horizon invariance therefore requires the policy-specific Q-CBF condition to hold throughout the claimed operating domain. We next give finite-horizon, high-probability safety guarantees for both stages of LIMBO.



\section{Safety Guarantees}
\label{sec:safety-guarantees}

We state two finite-horizon safety guarantees: one for the Stage~I
safety policy induced by the learned Q-CBF, and one for the Stage~II
task policy after the Q-CBF teacher has been internalized. 

\subsection{Stage-I Safety Guarantee}
Let
$\Omega$ denote the probability space for the randomness in training.
For each outcome $\omega\in\Omega$, the training procedure returns
learned functions $Q_\phi(\cdot,\cdot;\omega)$ and
$V_\phi(\cdot;\omega)$. Define $\mathcal G_Q\subseteq\Omega$ as the
event on which
\[
\sup_{x\in\cX,\,r\in\cR}
\bigl|Q_\phi(x,r;\omega)-Q^\star(x,r)\bigr|
\leq\varepsilon_Q,\quad\text{and}\quad
\sup_{x\in\cX}
\bigl|V_\phi(x;\omega)-V^\star(x)\bigr|
\leq\varepsilon_V.
\]

Assume $\Pr(\mathcal G_Q)\geq1-\delta_Q$ and that $\classkl$ is
$L_\classkl$-Lipschitz continuous on the relevant value range.
The same statements hold if the suprema in $\mathcal G_Q$ are
restricted to a verified reachable tube containing the closed-loop
trajectories.

\begin{theorem}[Stage-I high-probability safety]
\label{thm:main-stage-one-safety}
Let $R\sim\pi^\Ir(\cdot\,|\,x)$. Suppose there exist
$\Delta\geq\varepsilon_Q+L_\classkl\varepsilon_V$ and
$\delta\in[0,1]$ such that, for almost every training outcome and
every $x\in\cC^\star(\cXsafe)$, the bound holds:
\[
\Pr\bigl(
    Q_\phi(x,R)<\classkl\bigl(V_\phi(x)\bigr)+\Delta
\bigr)
\leq\delta,
\]
where the probability is over the action draw $R$, with the trained
functions and policy held fixed. At deployment, the conditional law
of $R_t$ given the training outcome and the history through $x_t$
is $\pi^\Ir(\cdot\mid x_t)$.
If $x_0\in\cC^\star(\cXsafe)$, then for every finite horizon $H$, the lower bound holds:
\begin{equation}
\label{eq:main-stage-one-safety}
\Pr\bigl(
    x_t\in\cC^\star(\cXsafe)
    \ \forall\,t=0,\ldots,H
\bigr)
\geq1-\delta_Q-H\delta.
\end{equation}
In particular, the probability of leaving the certified safe set is bounded as
$\Pr(\exists\,t\leq H:x_t\in\cXfail)
\leq\min\{1,\delta_Q+H\delta\}$.
Moreover, a deterministic Q-CBF filter that is feasible and enforces
the tightened learned condition throughout $\cC^\star(\cXsafe)$
yields $\delta=0$.
\end{theorem}
The proof is given in Appendix~\ref{app:hp-safety}.
Theorem~\ref{thm:main-stage-one-safety} gives a lower bound on the
probability of remaining in the certified safe set throughout $H$
control steps, accounting jointly for training and action sampling.
The tightening $\Delta$ covers the critic error $\varepsilon_Q$ and
the error $L_\classkl\varepsilon_V$ in the learned safety threshold.
Whenever the learned estimates satisfy their stated error tolerances,
an action meeting the tightened constraint therefore preserves the
certified safe set.

The term $\delta_Q$ bounds the probability that training returns
functions exceeding these tolerances somewhere in the claimed domain.
It enters once because the same learned functions are used along the
entire trajectory. In contrast, $\delta$ bounds the conditional
probability of violating the tightened constraint at each safe state,
so $H$ decisions contribute at most $H\delta$. Such a violation means
this sufficient safety condition is not met; it need not cause physical
failure. The bound requires no independence across time. A feasible
filter enforcing the tightened constraint removes the action-selection
contribution, leaving the probability bound $1-\delta_Q$.

The parameter $\delta$ also absorbs any lack of feasibility of the
tightened learned condition: if no residual satisfies the condition at a
state, its corresponding action-selection failure probability is one and
the bound is vacuous there.

\paragraph{Tolerances from finite-sample RL guarantees.}
The tolerances $\varepsilon_Q$ and $\varepsilon_V$ can be instantiated
from finite-sample guarantees for the Stage~I learning procedure. Finite-sample high-probability error bounds for learned value and
Q-functions have been established in a variety of RL settings, including
$\ell_\infty$ guarantees for discounted Q-learning
\citep{even2003learning,li2024qlearning} and finite-time guarantees for
fitted value iteration with function approximation
\citep{munos2008finite}. The precise rate and norm depend on the learning
algorithm, function class, and sampling and coverage assumptions. We
therefore state the connection to the Stage~I safety guarantee in terms of
a generic finite-sample critic bound.

Let $Q_\gamma^\star$ denote the exact state--action safety value associated
with the discounted Bellman recursion used during Stage~I at the final
discount factor $\gamma<1$. Suppose that, after $N$ sampled transitions,
the Stage~I critic admits a finite-sample guarantee of the form
\begin{equation}
\label{eq:finite-sample-stage-one-critic}
\Pr\left(
    \|Q_\phi-Q_\gamma^\star\|_{\infty,\mathcal Z}
    \leq
    \bar\varepsilon_Q(N,\delta_Q)
\right)
\geq
1-\delta_Q,
\end{equation}
where $\mathcal Z$ denotes the state--action domain on which the
certificate is claimed and $\bar\varepsilon_Q(N,\delta_Q)$ is the
corresponding finite-sample critic error. We use
\eqref{eq:finite-sample-stage-one-critic} abstractly: the safety result
does not require a particular RL sample-complexity theorem, only that an
appropriate high-probability error bound be available for the claimed
domain.
 The
specific dependence of $\bar\varepsilon_Q$ on $N$ depends on the
learning algorithm, function class, and coverage assumptions.

Since the safety guarantees below are stated relative to the
undiscounted viability value $Q^\star$, define the discount
approximation error
$
b_Q(\gamma)
:=
\|Q_\gamma^\star-Q^\star\|_{\infty,\mathcal Z}$.
Then, on the event in
\eqref{eq:finite-sample-stage-one-critic}, the triangle inequality gives us that 
\begin{equation}
\|Q_\phi-Q^\star\|_{\infty,\mathcal Z}
\leq
\bar\varepsilon_Q(N,\delta_Q)+b_Q(\gamma).
\end{equation}
Thus a valid choice for the critic tolerance in $\mathcal G_Q$ is
$
\varepsilon_Q
=
\bar\varepsilon_Q(N,\delta_Q)+b_Q(\gamma)$.
The first term is the finite-sample learning error, while
$b_Q(\gamma)$ accounts for the use of a discounted Bellman recursion
during Stage~I. A quantitative bound on $b_Q(\gamma)$ requires
additional assumptions on the viability problem.

The state-value tolerance $\varepsilon_V$ need not be obtained from an
independent learning bound. Recall that
\[
V^\star(x)
=
\max_{r\in\mathcal R}Q^\star(x,r),
\quad\text{and}\quad
V_\phi(x)
=
\E_{R\sim\pisafe_\psi(\cdot\mid x)}
\left[Q_\phi(x,R)\right].
\]
Define the Stage-I policy maximization error
\begin{equation}
\label{eq:stage-one-maximization-error}
\varepsilon_{\rm max}
:=
\sup_{x\in\mathcal Z_X}
\left\{
    \max_{r\in\mathcal R}Q_\phi(x,r)
    -
    \E_{R\sim\pisafe_\psi(\cdot\mid x)}
    Q_\phi(x,R)
\right\},
\end{equation}
where $\mathcal Z_X$ is the state projection of $\mathcal Z$. Whenever
$\|Q_\phi-Q^\star\|_{\infty,\mathcal Z}\leq\varepsilon_Q$, we have that
\begin{align}
V^\star(x)
&\leq
\max_{r\in\mathcal R}Q_\phi(x,r)+\varepsilon_Q
\leq
V_\phi(x)+\varepsilon_{\rm max}+\varepsilon_Q,
\end{align}
and, conversely, we have that
\begin{align}
V_\phi(x)
&\leq
\E_{R\sim\pisafe_\psi(\cdot\mid x)}
Q^\star(x,R)+\varepsilon_Q
\leq
V^\star(x)+\varepsilon_Q.
\end{align}
Therefore
$
\|V_\phi-V^\star\|_{\infty,\mathcal Z_X}
\leq
\varepsilon_Q+\varepsilon_{\rm max}$,
so a valid choice is
$
\varepsilon_V
=
\varepsilon_Q+\varepsilon_{\rm max}$.
As a result, a finite-sample Stage-I critic guarantee together with a
bound on the safety policy's maximization error supplies both
approximation tolerances appearing in $\mathcal G_Q$. In particular,
the tightening condition in
Theorem~\ref{thm:main-stage-one-safety} is satisfied whenever
\begin{equation}
\label{eq:finite-sample-stage-one-tightening}
\Delta
\geq
(1+L_\classkl)
\left[
    \bar\varepsilon_Q(N,\delta_Q)+b_Q(\gamma)
\right]
+
L_\classkl\varepsilon_{\rm max}.
\end{equation}
Hence an RL sample-complexity bound determines the amount of tightening
required by the subsequent high-probability safety guarantee.

\subsection{Stage-II High-Probability Internalized Safety}
Next we state a high probability safety guarantee for the Stage-II task policy wherein we have distilled the Q-CBF structure into the  policy during training. 

\begin{lemma}[Counterfactual correction loss implies a distillation tail bound]
\label{lem:counterfactual-distillation-tail}
Suppose the Q-CBF teacher is feasible on a domain $\mathcal D$, and define
$
\ell_{\rm cf}(x,r)
:=
1-
\exp\left(
-
\|r-\rteach_{\phi,\Delta}(x,r)\|_W^2
/\sigma_r^2
\right)$.
If
$
\sup_{x\in\mathcal D}
\E_{R\sim\pitask_\theta(\cdot\mid x)}
\left[
\ell_{\rm cf}(x,R)
\right]
\leq
\eta_\pi$,
then, for every $\varepsilon_\pi>0$, the estimate holds:
\begin{equation}
\label{eq:counterfactual-tail-bound}
\Pr\left(
\|R-\rteach_{\phi,\Delta}(x,R)\|_W
>
\varepsilon_\pi
\right)
\leq
\frac{
\eta_\pi
}{
1-\exp(-\varepsilon_\pi^2/\sigma_r^2)
}\quad\text{for every $x\in\mathcal D$}.
\end{equation}
\end{lemma}

\begin{proof}
On the event
$\|R-\rteach_{\phi,\Delta}(x,R)\|_W>\varepsilon_\pi$,
we have 
$\ell_{\rm cf}(x,R)
\geq
1-\exp(-\varepsilon_\pi^2/\sigma_r^2)$.
The result follows directly from Markov's inequality.
\end{proof}
Therefore we  set the distillation failure probability   to satisfy 
$
\delta_\pi
=
\min\left\{
1,
\eta_\pi/
(1-\exp(-\varepsilon_\pi^2/\sigma_r^2))
\right\}$.
The counterfactual Stage-II objective directly penalizes $\ell_{\rm cf}$, because the task proposal is executed rather than replaced by the correction. 
The distillation failure probability is used in the following theorem to quantify the probability of violating the learned Q-CBF condition during Stage-II deployment.
\begin{theorem}[Stage-II high-probability internalized safety]
\label{thm:main-stage-two-safety}
Suppose $Q^\star(x,\cdot)$ is $L_Q$-Lipschitz continuous on $\cR$,
uniformly over $\cC^\star(\cXsafe)$. Define the Q-CBF teacher by
\begin{equation}
\label{eq:main-stage-two-teacher}
\rteach_{\phi,\Delta}(x,r)
\in
\argmin_{
    \tilde r\in\cR
    }
\left\{\frac{1}{2}\|\tilde r-r\|_W^2\mid c_{\phi,\Delta}(x,\tilde r)\leq0\right\},
\end{equation}
Assume that, on $\mathcal G_Q$, the teacher problem
\eqref{eq:main-stage-two-teacher} is feasible for every
$x\in\cC^\star(\cXsafe)$. Moreover, for almost every realized training
outcome, let $R\sim\pitask_\theta(\cdot\mid x)$ and suppose
\[
\Pr\left(
    \|R-\rteach_{\phi,\Delta}(x,R)\|>\varepsilon_\pi
    \,\middle|\, Q_\phi,V_\phi,\pitask_\theta
\right)
\leq\delta_\pi
\]
for every $x\in\cC^\star(\cXsafe)$.
%
Set
$\Delta\geq\varepsilon_Q+L_\classkl\varepsilon_V+L_Q\varepsilon_\pi$.
For every $x_0\in\cC^\star(\cXsafe)$ and finite horizon $H$,
\begin{equation}
\label{eq:main-stage-two-safety}
\Pr\bigl(
    x_t\in\cC^\star(\cXsafe)
    \ \forall\,t=0,\ldots,H
\bigr)
\geq1-\delta_Q-H\delta_\pi.
\end{equation}
In particular,
$\Pr(\exists\,t\leq H:x_t\in\cXfail)
\leq\min\{1,\delta_Q+H\delta_\pi\}$.
\end{theorem}
The full proof is given in Appendix~\ref{app:hp-safety}.

\begin{remark}[Robust certification near the viability boundary]
\label{rem:zero-margin-boundary}
The tightening $\Delta$ accounts for certificate and distillation errors, but also requires sufficient Q-CBF slack. For $x\in\cC^\star(\cXsafe)$, the maximum exact slack is
\[
\max_{r\in\cR}\left\{Q^\star(x,r)-\classkl\bigl(V^\star(x)\bigr)\right\}
=
V^\star(x)-\classkl\bigl(V^\star(x)\bigr).
\]
At states with $V^\star(x)=0$, this slack is zero. The exact Q-CBF condition remains feasible, but there is no positive slack to absorb approximation and distillation errors. Consequently, the tightening required by the error bounds may leave no feasible action.

When the tightened condition is infeasible, Theorem~\ref{thm:main-stage-one-safety} gives no positive lower bound on the probability of remaining safe, and the feasibility assumption of Theorem~\ref{thm:main-stage-two-safety} is not satisfied. The controller may still remain safe, but these results do not certify its behavior there. Related limitations appear in stochastic CBF analysis, where finite-horizon safety bounds can become uninformative at the safe set boundary~\citep{cosner2023robustsafetystochasticuncertainty}.
\end{remark}

\section{Experiments}
\label{sec:experiments}

We run two experiments, dodgeball and locomotion beneath obstacles, to evaluate high-dimensional Q-CBF synthesis, transfer of learned safety to hardware, and  risk-guided  behavior discovery. Dodgeball is a quantitative comparison of LIMBO and CBF-RL~\cite{yang2026pacmanperceptionawarecbfrlwholebody}, and tests sim-to-real transfer under dynamic collision avoidance. Locomotion beneath obstacles  evaluates coordinated whole-body safety  and isolates the mechanism by which replay concentration induces strategies. In the comparisons below, QCBF-PD denotes the Stage-II task policy trained using Algorithm~\ref{alg:qcbf-distillation}.

\subsection{Common Experimental Setup}
\label{sec:common-exp-setup}

All experiments use a Unitree G1 humanoid with $29$ actuated degrees of freedom. Both training stages are performed in simulation using $4096$ parallel environments, with policies operating at $50$~Hz. The task policy produces whole-body joint residuals around a frozen base controller, which provides standing balance for dodgeball and forward locomotion for the low-obstacle task. Stage II uses PPO~\cite{schulman2017proximal} with Q-CBF feedback, task and control rewards, and an adversarial motion prior~\cite{amp2021} constructed from Kimodo-generated reference motions~\cite{RempeKimodo2026}.

The common safety specification is $g(x)=\min\{h_{\mathrm{collision}}(x),h_{\mathrm{tilt}}(x),h_{\mathrm{ground}}(x)\}$, combining obstacle clearance, pelvis orientation, and trunk height above the ground. Each component is normalized so that zero denotes its failure boundary. The collision component is task dependent, measuring clearance from the ball or overhead obstacle. Training includes parameter randomization and observation noise. On hardware, the learned task policy is composed with the frozen base through the same residual interface used in simulation; the learned safety critic and Stage-I safety policy are not required at deployment. Additional details on the simulator, observations, safety margins, training hyperparameters, domain randomization, and evaluation procedures are provided in Appendix~\ref{app:additional-experiments}.

\begin{figure}[t]
    \centering
        \includegraphics[width=0.65\columnwidth]{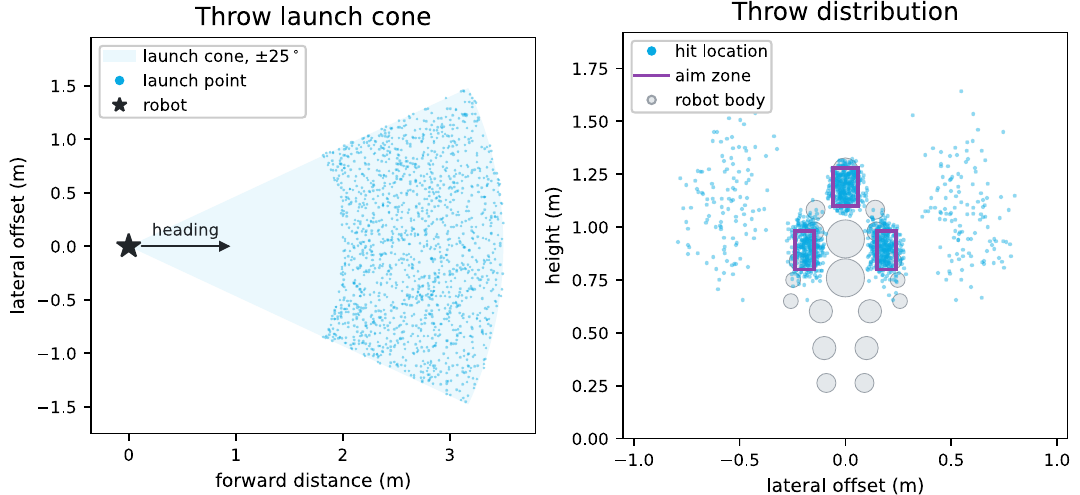}
    \caption{\textbf{Dodgeball training distribution.} \textbf{\emph{(left)}} Training throws span a $\pm25^\circ$ frontal cone and target the head and shoulders, with $15\%$ sampled misses. \textbf{\emph{(right)}} \textcolor{Purple}{Purple} regions mark the aim zones, and \textcolor{bluefire}{blue} indicates hit locations.}
    \label{fig:dodge-distribution}
\end{figure}

\subsection{Dodgeball Avoidance from Simulation to Hardware}
\label{sec:showcase-dodge}

Dodgeball  requires the robot to avoid incoming balls while maintaining balance and remaining near its initial position~\cite{yang2026pacmanperceptionawarecbfrlwholebody}.
The task policy receives current proprioceptive observations and a four-frame history of ball positions, without an explicit ball-velocity input, and outputs a 29-dimensional whole-body residual.

We report hit rate and fall rate as the primary safety outcomes \cite{yang2026pacmanperceptionawarecbfrlwholebody}. In simulation, we additionally measure torque saturation, while on hardware we measure station drift and peak body tilt to quantify the magnitude of the whole-body evasive response. Figure~\ref{fig:dodge-distribution} summarizes the training distribution. Launch points span a $\pm25^\circ$ frontal cone, with variation in launch position and flight time. We use a frontal cone to reflect the intended onboard-perception setting, where a forward-facing depth camera  provides ball observations. Training targets the head and shoulders and includes $15\%$ misses.

\subsubsection{Learned Safety Value}

We first test whether the learned state value $V_\phi$ captures encounter difficulty by evaluating it at ball launch (Figure~\ref{fig:barrier-value}). In the left panel, approximately $65\%$ of throws that result in hits have $V_\phi\leq0.2$, compared with $24\%$ of evaded throws. The right panel groups throws by launch-time $V_\phi$: hit rates generally decrease as $V_\phi$ increases for both the Stage-I safety policy and QCBF-PD. Together, these results support $V_\phi$ as an empirical ordering of encounter difficulty. QCBF-PD also achieves substantially lower hit rates than the Stage-I safety policy across the observed value range, showing that Stage~II improves avoidance performance rather than merely reproducing the safety policy.

\begin{figure}[t]
    \centering
    \includegraphics[width=0.8\columnwidth]{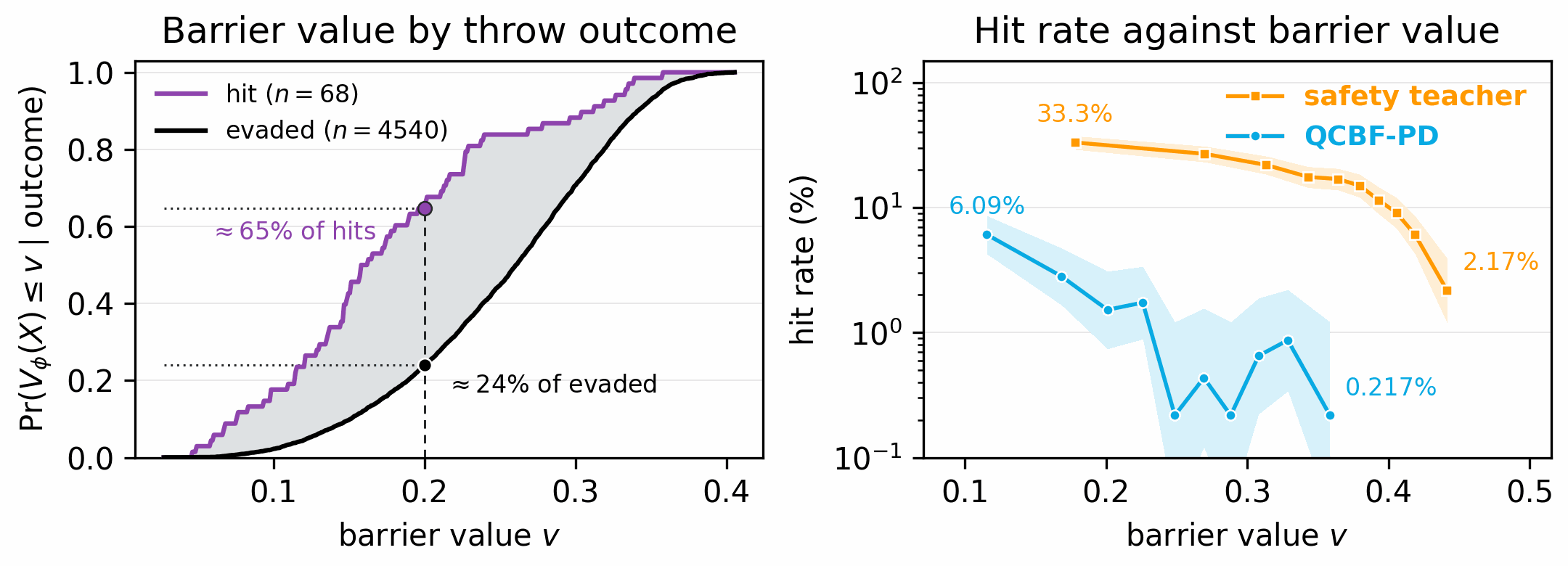}
    \caption{\textbf{Barrier value and hit rate.} \textbf{\emph{(left)}} Cumulative distributions of launch-time safety values for hits and evaded throws; hits tend to originate from states with lower safety values. \textbf{\emph{(right)}} Estimated hit probabilities $\Pr_\pi(\mathrm{hit}\mid V_\phi(X)\in[v_k,v_{k+1}))$ in ten bins, with 95\% confidence intervals and points placed at each bin's mean safety value. Hit rates generally decrease with launch-time $V_\phi$ and are lower for QCBF-PD than for the Stage-I safety policy.}
      \label{fig:barrier-value}
\end{figure}

\begin{figure}[t]
    \centering
    \includegraphics[width=0.85\linewidth]{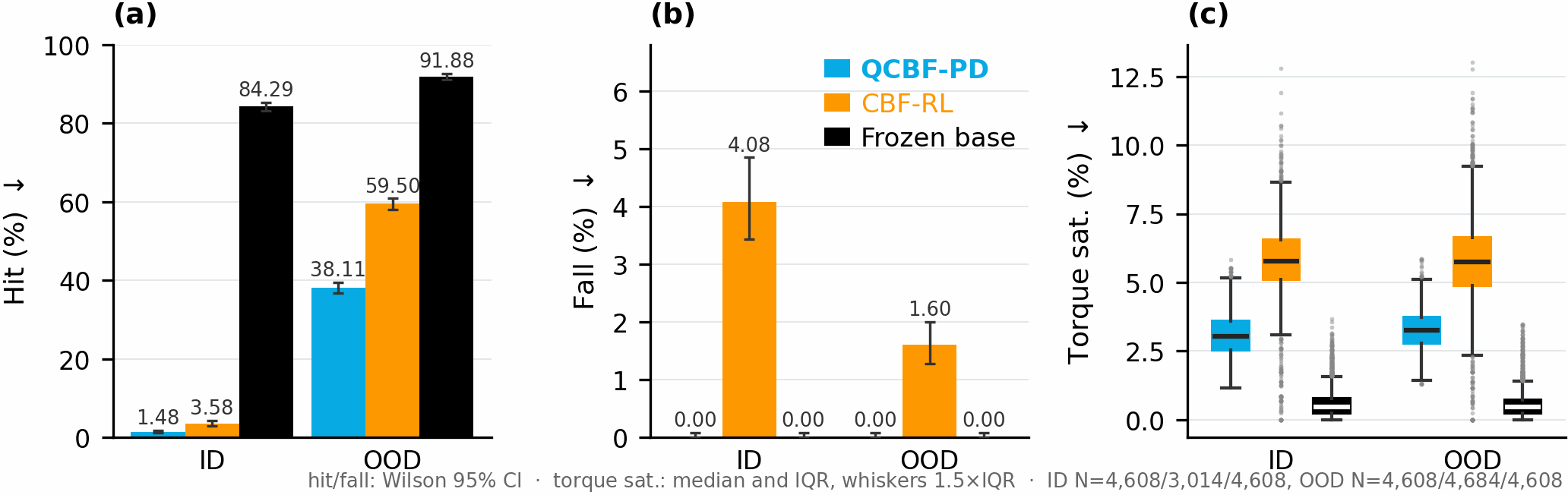}
    \caption{\textbf{Simulated dodgeball.} \textbf{(a)} Hit rate, \textbf{(b)} fall rate, and \textbf{(c)} torque saturation; lower is better for all metrics.
\textbf{\textcolor{bluefire}{QCBF-PD}} is hit less than half as often as \textbf{\textcolor{vitaminc}{CBF-RL}} on in-distribution throws and has no observed falls in either setting.
\textbf{\textcolor{vitaminc}{CBF-RL}} falls on 4.08\% and 1.60\% of in- and out-of-distribution throws, respectively.
\textbf{\textcolor{bluefire}{QCBF-PD}} also has roughly half the median torque saturation of \textbf{\textcolor{vitaminc}{CBF-RL}} in both settings, while the frozen base controller has low torque saturation but high hit rates.}
\label{fig:tableII-simdodgeball}
\end{figure}

\subsubsection{Simulation Results}

We compare QCBF-PD with the analytical CBF-RL controller of~\citet{yang2026pacmanperceptionawarecbfrlwholebody} and the frozen base controller. Figure~\ref{fig:tableII-simdodgeball} reports performance both on the trained upper-body aim zones and on an out-of-distribution (OOD) set targeting previously unseen body regions---the chest, torso, pelvis, hips, and thighs (cf.~Figure~\ref{fig:hardware-throws}).

On in-distribution throws, QCBF-PD is hit less than half as often as CBF-RL ($1.48\%$ vs.~$3.58\%$) and incurs no falls, whereas CBF-RL falls on $4.08\%$ of throws.  The separation becomes larger OOD: QCBF-PD is hit on $38.11\%$ of throws compared with $59.50\%$ for CBF-RL, while QCBF-PD again incurs no falls and CBF-RL falls on $1.60\%$ of throws. The hit- and fall-rate differences are statistically significant: indeed for hit, the in-distribution significance is $p=4.5\scinot{9}$ and the out-of-distribution significance is $p=4.4\scinot{95}$, while for fall, the in-distribution significance is $p=5.9\scinot{51}$ and the out-of-distribution significance is $p=4.7\scinot{23}$.  Improved robustness does not require more aggressive actuation. QCBF-PD's torque distribution changes little between the trained and OOD sets, with medians $\approx 3.0\%$ and $3.3\%$, and remains roughly half that of CBF-RL ($5.8\%$ and $5.7\%$). The torque-saturation distributions differ significantly between QCBF-PD and CBF-RL in both the in-distribution and OOD settings (Mann–Whitney U tests).

\begin{figure}[t]
    \centering
    \includegraphics[width=0.85\columnwidth]{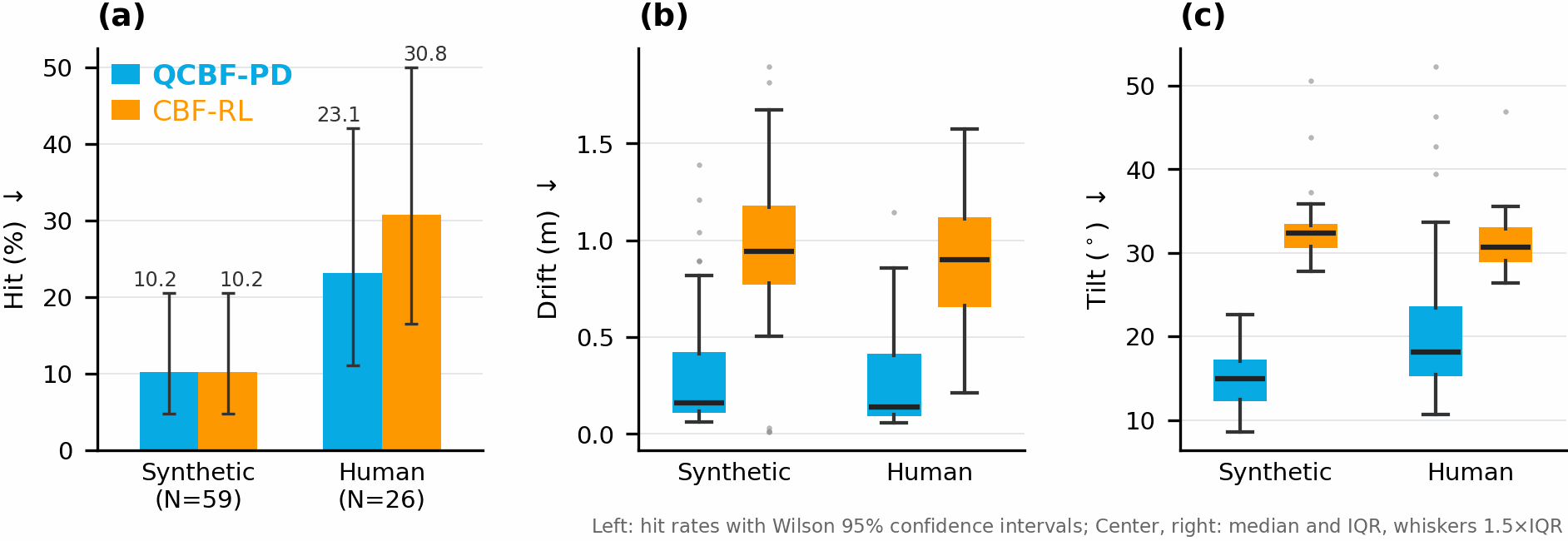}
    \caption{\textbf{Hardware dodgeball.} \textbf{(a)} Hit rates do not differ significantly between controllers in either setting. Both have a $10.2\%$ hit rate on synthetic throws (exact McNemar test on paired throws, $p=1.0$). On human throws, \textbf{\textcolor{bluefire}{QCBF-PD}} and \textbf{\textcolor{vitaminc}{CBF-RL}} have hit rates of $23.1\%$ and $30.8\%$, respectively (Fisher's exact test, $p=0.76$). \textbf{(b)} QCBF-PD has lower median drift than CBF-RL: $0.16$\,m versus $0.94$\,m on synthetic throws (Wilcoxon signed-rank test, $p=7\times10^{-10}$) and $0.14$\,m versus $0.90$\,m on human throws (Mann--Whitney $U$ test, $p=2.1\times10^{-6}$). \textbf{(c)} Median peak tilt is also lower: $14.9^\circ$ versus $32.4^\circ$ on synthetic throws (Wilcoxon signed-rank test, $p=2.4\times10^{-11}$) and $18.2^\circ$ versus $30.7^\circ$ on human throws (Mann--Whitney $U$ test, $p=4.6\times10^{-4}$).}
    \label{fig:tableIII-humansynth-db}
\end{figure}

\subsubsection{Hardware Results}

We evaluate the distilled policies on hardware in two complementary settings. Synthetic trajectories isolate transfer of the robot's evasive response under repeatable obstacles, while human throws introduce variation in the throw, ball tracking, and physical interaction.

\noindent\textbf{Synthetic trajectories.}
We replay the same $59$ synthetic ball trajectories for both controllers. The robot uses real proprioceptive observations together with the scripted ball positions. Figure~\ref{fig:tableIII-humansynth-db} reports the same $10.2\%$ hit rate for both methods. QCBF-PD has lower median drift ($0.162$ vs.\ $0.944$~m) and tilt ($14.9^\circ$ vs.\ $32.4^\circ$), indicating smaller positional and angular excursions at the same observed avoidance rate.

\noindent\textbf{Human throws.}
We next evaluate human-thrown balls tracked with motion capture.
Figure~\ref{fig:hardware-throws} reports the spatial distribution of the physical throws in addition to their outcomes, making the coverage of the hardware evaluation explicit. This information is not reported in the PAC-MAN hardware evaluation~\cite{yang2026pacmanperceptionawarecbfrlwholebody}, which reports only the overall success rate for hand-thrown trials. QCBF-PD evades $20$ of $26$ throws, compared with $18$ of $26$ for CBF-RL, corresponding to hit rates of $23.1\%$ and $30.8\%$. At this sample size the hit-rate difference is not statistically significant, while QCBF-PD exhibits substantially smaller drift and body tilt as shown in Figure~\ref{fig:tableIII-humansynth-db}.

\begin{figure}[t]
    \centering
    \includegraphics[width=0.65\columnwidth]{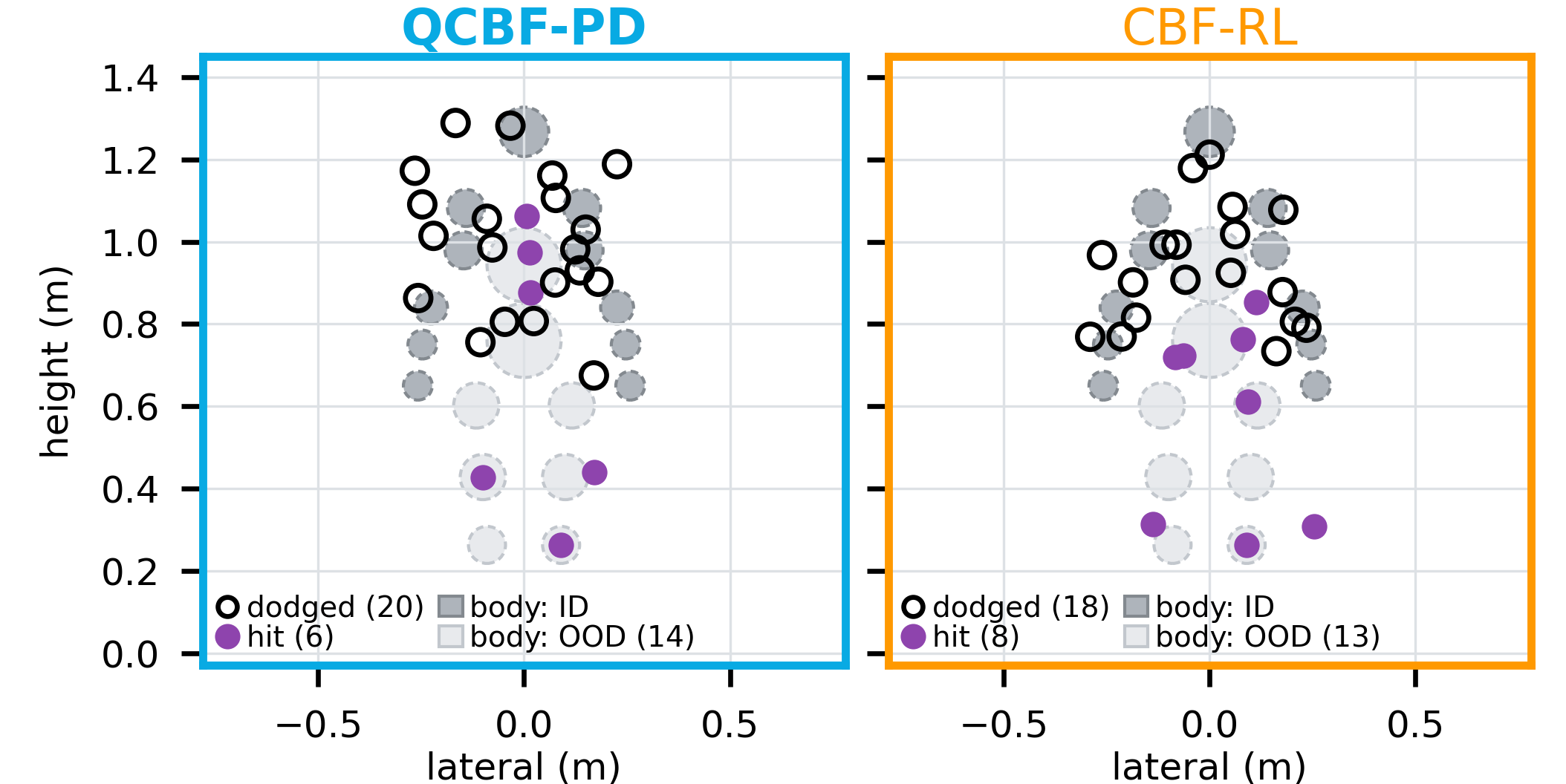} 
    \caption{\textbf{Hardware dodgeball with human throws.}
    Throw locations and outcomes for
    \textbf{\textcolor{bluefire}{QCBF-PD}} (left) and
    \textbf{\textcolor{vitaminc}{CBF-RL}} (right).
    Open circles indicate evaded throws, and purple circles indicate hits.
    QCBF-PD evades $20$ of $26$ throws, compared with
    $18$ of $26$ for CBF-RL.}
    \label{fig:hardware-throws}
\end{figure}

\begin{figure}[t]
    \centering
   \includegraphics[width=0.9\linewidth]{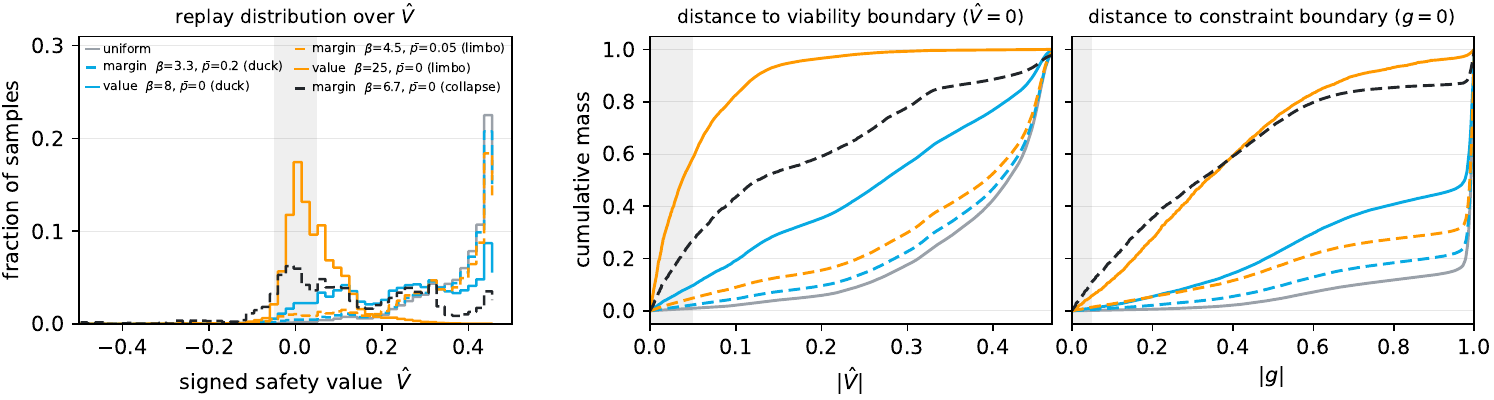}
\caption{\textbf{Limbo: Value-guided replay concentrates learning near the safety boundary.}
\textbf{\emph{(left)}}: Replay distribution over the signed safety value $\hat V$. Sampling weights are proportional to $\exp(-\beta|\hat V|)$ for value-guided replay and $\exp(-\beta|g|)+\bar p$ for margin-guided replay.
\textbf{\emph{(center, right)}}: Cumulative replay mass versus the absolute safety value $|\hat V|$ and absolute safety margin $|g|$, respectively. Gray bands highlight small absolute values around $\hat V=0$ and $g=0$, corresponding to the estimated viability and instantaneous constraint boundaries.
Uniform replay is shown in gray.
Solid curves indicate value-guided replay: blue for duck~($\beta=8$) and orange for limbo~($\beta=25$).
Dashed curves indicate margin-guided replay: blue for duck~($\beta=3.3$), orange for limbo~($\beta=4.5$), and black for collapse~($\beta=6.7$).}
    \label{fig:boundary-replay}
\end{figure}

\subsection{LIMBO: Locomotion Under Low Obstacles}
\label{sec:showcase-limbo}
\label{sec:limbo}

The second task is forward locomotion beneath a  horizontal
obstacle (aka limbo). The frozen base controller continues to command forward walking, while the learned residual must alter the whole-body motion sufficiently to clear the obstacle without contact or loss of balance. The same state-based safety structure is used as in dodgeball, with $h_{\mathrm{collision}}$ now measuring clearance from the overhead obstacle. Importantly, the specification does not prescribe how the robot should clear the bar: both lowering the body and leaning backward are compatible with safety. We evaluate cleared, struck and aborted rates, and the replay effect on the synthesized behavior in simulation. We  deploy the QCBF-PD policy on hardware, and the results can be seen on the project website. 

\begin{figure}[t]
    \centering
    \includegraphics[width=0.8\columnwidth]{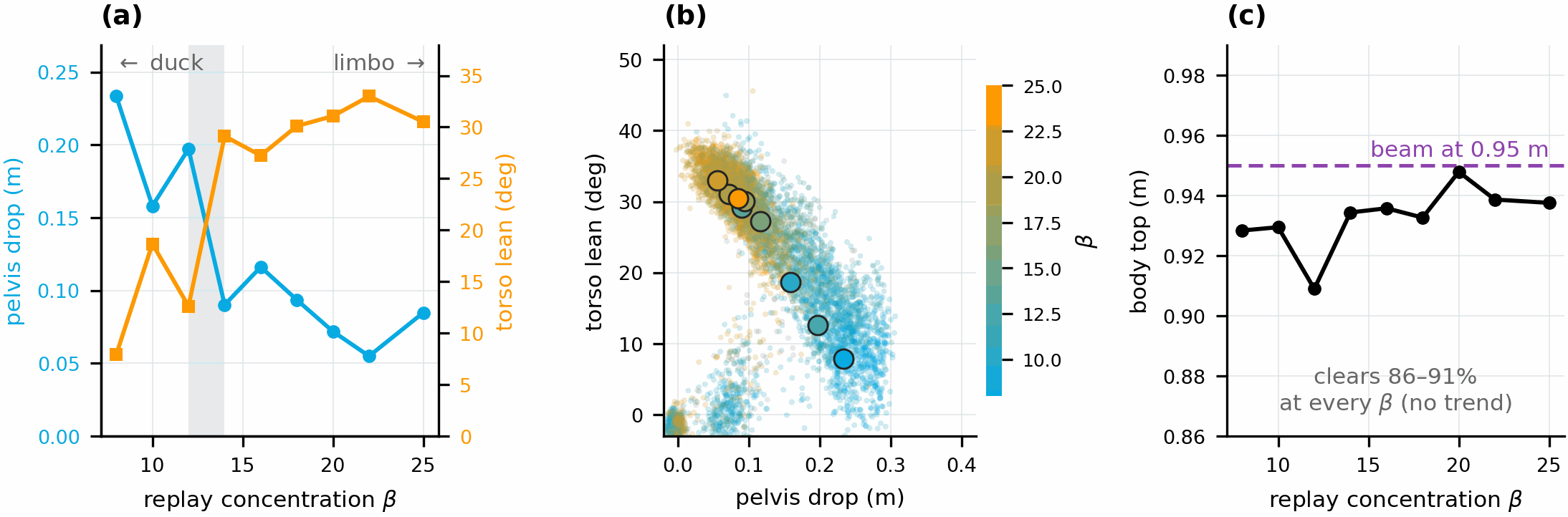}
\caption{\textbf{Replay concentration shapes avoidance behavior.} \textbf{(a)} \textbf{\textcolor{bluefire}{Pelvis drop}} and \textbf{\textcolor{vitaminc}{torso lean}} versus replay concentration~$\beta$; the shaded band marks the transition from ducking to backward leaning. \textbf{(b)} Pelvis drop and torso lean for each scored wall, with large markers indicating the medians for each~$\beta$. \textbf{(c)} Body-top height remains within $0.909$--$0.948$\,m across the sweep, while clearance rates show no clear trend.}
\label{fig:beta-sweep}
\end{figure}

\paragraph{Learned Safety Value.}
We examine how the learned safety value shapes replay during synthesis.
Figure~\ref{fig:boundary-replay} (left) shows the selected sample distribution
over signed safety values: uniform replay and weak concentration favor
comfortably positive values, while stronger value-guided replay shifts
mass toward $\hat V=0$. Figure~\ref{fig:boundary-replay}  (center, right) quantifies this
shift through cumulative mass near the estimated viability and
instantaneous constraint boundaries. Increasing value-guided concentration
from $\beta=8$ to $\beta=25$ raises the selected mass within
$|\hat V|\leq0.05$ from roughly $10\%$ to $60\%$, and also shifts samples
closer to $g=0$. Concentrating critic and safety-policy updates where
recoverability is nearly lost targets more precise learning of the
safety boundary. The zeroth-order, margin-guided sampler uses only $g$
and cannot distinguish states with the same instantaneous margin but
different future recoverability. It also produces limbo at one setting,
but its most aggressive setting shown collapses, whereas value-guided
replay at $\beta=25$ supports stable limbo. These results illustrate how
value-guided sampling can expose viable motions whose safety is not
obvious \emph{a priori}, under the same safety specification.

\begin{figure}[t]
    \centering
    \includegraphics[width=0.7\columnwidth]{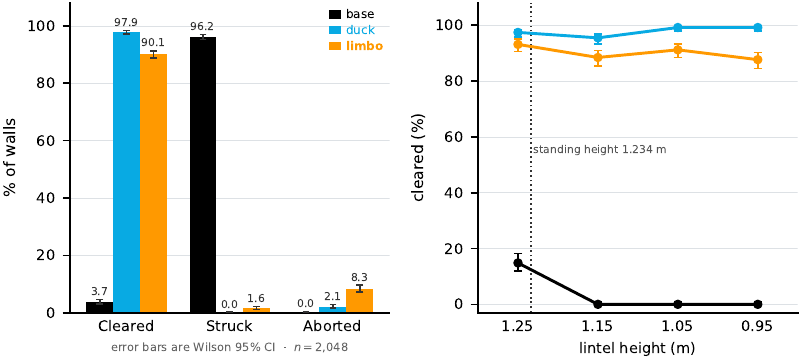}
\caption{\textbf{Clearance outcomes and dependence on obstacle height.} \textbf{\emph{(left)}} Duck and limbo both achieve high clearance rates. Limbo's $8.3\%$ abort rate is entirely due to lane exits, with no observed topples. \textbf{\emph{(right)}} The dotted line marks standing head height ($1.234$\,m). The frozen base policy clears no obstacles below this height, while duck and limbo maintain $88$--$99\%$ clearance across all tested heights.}
\label{fig:limbo-outcomes}
\end{figure}

\paragraph{Simulation Evaluation.}
Clearing a beam below standing head height ($1.234$\,m) forces the robot to lower its highest point: it can bend the knees and drop the pelvis, or tip the torso back.  Sweeping replay concentration $\beta$ over the Stage-I synthesis (via the value-based sampling rule $p \propto e^{-\beta|\hat V|}$) changes the balance between pelvis drop and torso lean (Figure~\ref{fig:beta-sweep}, left). Across the sweep the pelvis drop falls by $\approx 4\times$ ($0.233 \to 0.055$\,m) while the torso lean rises by  $\approx 4\times$ ($7.9 \to 33.0^\circ$), with knee flexion tracking the pelvis drop. The dominant mechanism flips for $\beta\in(12,14)$.  The per-wall scatter shows a tradeoff between pelvis lowering and torso lean (Figure~\ref{fig:beta-sweep}, center). The value of $\beta$ changes which states dominate learning: strong boundary concentration targets the edge of recoverability, exposing states in which substantial backward tilt remains viable (cf.~Figure~\ref{fig:beta-sweep}, left and center). Additionally, clearance rates range from $85.7\%$ to $91.2\%$, with no clear monotonic trend across the sweep (Figure~\ref{fig:beta-sweep}, right).

The two policies deployed on hardware exhibit these contrasting strategies: one ducks beneath the obstacle, while the other uses a backward lean. Measured over $2{,}048$ walls each, the duck and limbo policies lower the body top by the same amount ($0.285$ vs.\ $0.289$\,m) with  $11.2^\circ$ vs.\ $30.5^\circ$ of pelvis tilt---i.e., same clearance, different margin. The deeper lean comes at a cost in station-keeping: limbo drifts $0.775$\,m compared to duck's $0.392$\,m, with a correspondingly larger heading error ($27.5^\circ$ vs.\ $3.7^\circ$). This is the honest cost of the maneuver, not a stability failure; across all walls the limbo never topples, and its $8.3\%$ abort rate (Figure~\ref{fig:limbo-outcomes} (left)) is composed entirely of lane exits. Both deployed policies hold $88$--$99\%$ clearance at every beam height, with headroom to spare at the lowest rung, while the frozen base policy clears $14.8\%$ at $1.25$\,m and $0\%$ at every lower rung (Figure~\ref{fig:limbo-outcomes} (right)). These results demonstrate that LIMBO can discover distinct whole-body strategies under the same safety specification without prescribing a particular avoidance motion.

\section{Conclusion}
We present LIMBO, a scalable framework for synthesizing a safety certificate from black-box transitions and distilling its guidance into a task policy. The residual Q-CBF formulation combines risk-guided boundary exploration with counterfactual feedback during policy training, placing safety synthesis and task learning in the same control space. We characterize conditions for safety transfer under approximation and distillation errors and provide finite-horizon, high-probability guarantees for both stages. 

Our experiments in humanoid dodgeball and limbo demonstrate safe, agile behavior on hardware without online filtering. We further show that, under the same safety specification, varying the boundary-sampling concentration produces distinct strategies, from crouching to backward-leaning limbo. Future work includes extending LIMBO to other embodiments (e.g., dexterous manipulation) and further investigating how risk-guided boundary sampling shapes safe learning and behavior discovery.





%
%

\bibliographystyle{plainnat}
\bibliography{refs}


\newpage
\appendix

\section{Proof of Infinite-Horizon Viability Guarantee}\label{app:viability-proof}
Here we prove the infinite-horizon viability guarantee. 

\begin{proof}[Proof of Proposition~\ref{prop:viability-value}]
We prove the result in several steps.

\paragraph{Step 1: finite-horizon values are continuous and satisfy the
Bellman recursion.}
For fixed $x$, the finite-horizon trajectory depends continuously on
the input sequence in the compact set $\cU^H$. Thus the maximum in
the definition of $V_H(x)$ is attained. Splitting a length-$(H+1)$
sequence into its first action and its remaining $H$ actions gives
\[
V_{H+1}(x)
=\max_{u\in\cU}\min\{\marginfn(x),V_H(f(x,u))\}
=\min\{\marginfn(x),\max_{u\in\cU}V_H(f(x,u))\}.
\]
Here maximization commutes with taking the minimum with the fixed
scalar $\marginfn(x)$ because the maxima are attained. Consequently,
\[
V_0=\marginfn,
\qquad
V_{H+1}=\mathcal T V_H,
\]
where
\[
(\mathcal T V)(x)
=
\min\left\{
\marginfn(x),
\max_{u\in\cU}V\bigl(f(x,u)\bigr)
\right\}.
\]
We first show by induction that $V_H$ is continuous for every finite
$H$. The base case $V_0=\marginfn$ is continuous by
Assumption~\ref{ass:viability-regularity}. Suppose that $V_H$ is
continuous. Then
\[
(x,u)\longmapsto V_H\bigl(f(x,u)\bigr)
\]
is continuous because both $f$ and $V_H$ are continuous.

Define
\[
M_H(x)
:=
\max_{u\in\cU}V_H\bigl(f(x,u)\bigr).
\]
Since $\cU$ is compact, the maximum is attained for every $x$.
Moreover, $M_H$ is continuous. To see this directly, let
$x_n\to x$. For a maximizer
\[
u^\star\in
\argmax_{u\in\cU}V_H\bigl(f(x,u)\bigr),
\]
we have
\[
M_H(x_n)
\geq
V_H\bigl(f(x_n,u^\star)\bigr),
\]
and hence, by continuity,
\[
\liminf_{n\to\infty}M_H(x_n)
\geq
V_H\bigl(f(x,u^\star)\bigr)
=
M_H(x).
\]
For the reverse inequality, choose
\[
u_n\in
\argmax_{u\in\cU}V_H\bigl(f(x_n,u)\bigr).
\]
By compactness of $\cU$, every subsequence of $\{u_n\}$ has a further
subsequence, denoted again by $\{u_n\}$, converging to some
$\bar u\in\cU$. Along such a subsequence,
\[
M_H(x_n)
=
V_H\bigl(f(x_n,u_n)\bigr)
\longrightarrow
V_H\bigl(f(x,\bar u)\bigr)
\leq
M_H(x).
\]
Thus
\[
\limsup_{n\to\infty}M_H(x_n)\leq M_H(x).
\]
Combining the two inequalities shows that $M_H$ is continuous.
Therefore
\[
V_{H+1}(x)
=
\min\{\marginfn(x),M_H(x)\}
\]
is continuous. By induction, $V_H\in C(\cX)$ for every finite $H$.

\paragraph{Step 2: the sequence $\{V_H\}$ decreases pointwise and has a
finite limit.}
The Bellman operator is monotone: if $V\leq W$ pointwise, then
$
\mathcal T V\leq\mathcal T W$.
Indeed, we have 
\[
V\bigl(f(x,u)\bigr)
\leq
W\bigl(f(x,u)\bigr)
\qquad
\forall (x,u),
\]
so that 
\[
\max_{u\in\cU}V\bigl(f(x,u)\bigr)
\leq
\max_{u\in\cU}W\bigl(f(x,u)\bigr),
\]
and taking the minimum with the common term $\marginfn(x)$ preserves
the inequality.

Moreover, we have
$
\mathcal T\marginfn\leq\marginfn$,
since the first argument of the minimum defining
$\mathcal T\marginfn$ is $\marginfn(x)$. Hence
\[
V_1
=
\mathcal T V_0
=
\mathcal T\marginfn
\leq
\marginfn
=
V_0.
\]
Applying monotonicity repeatedly yields the bound
\[
V_{H+1}
=
\mathcal T V_H
\leq
\mathcal T V_{H-1}
=
V_H
\qquad
\forall H\geq1.
\]
Thus $\{V_H(x)\}_{H\geq0}$ is nonincreasing for every fixed
$x\in\cX$.

Since $\marginfn$ is bounded, there exists
$\underline g\in\mathbb R$ such that
$
\underline g
\leq
\marginfn(x)$
for all $x\in\cX$.
From the finite-horizon definition,
\[
V_H(x)
=
\max_{\mathbf u\in\cU^H}
\min_{0\leq k\leq H}\marginfn(x_k),
\]
and therefore
$
V_H(x)\geq\underline g$.
Hence, for every $x$, the nonincreasing sequence
$\{V_H(x)\}$ is bounded below and consequently converges. Define
\[
V^\star(x)
:=
\lim_{H\to\infty}V_H(x)
=
\inf_{H\geq0}V_H(x).
\]
For every $x\in\cX$ and $H\geq0$, the preceding bounds give
\[
\underline g
\leq V^\star(x)
\leq V_H(x)
\leq V_0(x)
=\marginfn(x).
\]
Thus $V^\star(x)\in\mathbb R$ for every $x$, and $V^\star$ is
bounded on $\cX$ because $\underline g$ is a uniform lower bound
and $\marginfn$ is bounded above.

\paragraph{Step 3: $V^\star$ is upper semicontinuous.}
Since each $V_H$ is continuous and
\[
V^\star=\inf_{H\geq0}V_H,
\]
the function $V^\star$ is upper semicontinuous. Indeed, for any
$a\in\mathbb R$, we have 
\[
\{x:V^\star(x)<a\}
=
\bigcup_{H\geq0}\{x:V_H(x)<a\},
\]
which is open because every $V_H$ is continuous. Therefore
$V^\star\in\mathrm{USC}_b(\cX)$ as claimed.

\paragraph{Step 4: the Bellman maximization passes to the infinite-horizon
limit.}
Fix $x\in\cX$, and define
\[
m_H(x)
:=
\max_{u\in\cU}V_H\bigl(f(x,u)\bigr).
\]
Since $V_{H+1}\leq V_H$ and $V_H\geq\underline g$, the sequence
$\{m_H(x)\}$ is nonincreasing and bounded below by $\underline g$.
It therefore has a finite limit; write
$
m_\infty(x)
:=
\lim_{H\to\infty}m_H(x)$.
Since $V^\star\leq V_H$ pointwise,
\[
\max_{u\in\cU}V^\star\bigl(f(x,u)\bigr)
\leq
m_H(x)
\qquad
\text{for all}\; H,
\]
and hence
\begin{equation}
\label{eq:viability-max-limit-lower}
\max_{u\in\cU}V^\star\bigl(f(x,u)\bigr)
\leq
m_\infty(x).
\end{equation}
The maximum on the left is attained because
$V^\star$ is upper semicontinuous, $f(x,\cdot)$ is continuous,
and $\cU$ is compact.

We now prove the reverse inequality. For each $H$, choose a
maximizer
\[
u_H\in
\argmax_{u\in\cU}V_H\bigl(f(x,u)\bigr),
\]
so that
$
m_H(x)=V_H\bigl(f(x,u_H)\bigr)$.
By compactness of $\cU$, there exists a subsequence
$\{u_{H_j}\}_{j\geq1}$ converging to some $\bar u\in\cU$.
Since $m_H(x)$ itself converges, the same subsequence satisfies
$
m_{H_j}(x)\longrightarrow m_\infty(x)$.

Fix any finite $K$. For all sufficiently large $j$,
$H_j\geq K$. Since the value sequence is decreasing,
$
V_{H_j}\leq V_K$,
and therefore
\[
m_{H_j}(x)
=
V_{H_j}\bigl(f(x,u_{H_j})\bigr)
\leq
V_K\bigl(f(x,u_{H_j})\bigr).
\]
Taking $j\to\infty$ and using continuity of $V_K$ and $f$ gives
\[
m_\infty(x)
\leq
V_K\bigl(f(x,\bar u)\bigr).
\]
Since this holds for every $K$,
\begin{align}
m_\infty(x)
&\leq
\inf_{K\geq0}
V_K\bigl(f(x,\bar u)\bigr) =
V^\star\bigl(f(x,\bar u)\bigr) \leq
\max_{u\in\cU}
V^\star\bigl(f(x,u)\bigr).
\label{eq:viability-max-limit-upper}
\end{align}
Combining
\eqref{eq:viability-max-limit-lower} and
\eqref{eq:viability-max-limit-upper} yields
\begin{equation}
\label{eq:viability-max-limit}
\lim_{H\to\infty}
\max_{u\in\cU}V_H\bigl(f(x,u)\bigr)
=
\max_{u\in\cU}V^\star\bigl(f(x,u)\bigr).
\end{equation}

\paragraph{Step 5: $V^\star$ is a fixed point of $\mathcal T$.}
Using the finite-horizon Bellman recursion,
\[
V_{H+1}(x)
=
\min\left\{
\marginfn(x),
\max_{u\in\cU}V_H\bigl(f(x,u)\bigr)
\right\}.
\]
Taking $H\to\infty$ and applying
\eqref{eq:viability-max-limit} gives
\begin{align*}
V^\star(x)
&=
\min\left\{
\marginfn(x),
\max_{u\in\cU}
V^\star\bigl(f(x,u)\bigr)
\right\}=
(\mathcal T V^\star)(x).
\end{align*}
Since $x$ was arbitrary, we have
$
\mathcal T V^\star=V^\star$.

\paragraph{Step 6: $V^\star$ is the greatest bounded upper
semicontinuous fixed point.}
Let
$W\in\mathrm{USC}_b(\cX)$ be any fixed point of $\mathcal T$:
\[
\mathcal T W=W.
\]
From the definition of $\mathcal T$, we have that 
\[
W(x)
=
\min\left\{
\marginfn(x),
\max_{u\in\cU}W\bigl(f(x,u)\bigr)
\right\}
\leq
\marginfn(x),
\]
so that 
$
W\leq V_0$.
By monotonicity of $\mathcal T$, we have that 
\[
\mathcal T^H W
\leq
\mathcal T^H V_0
=
V_H
\qquad
\text{for all}\;\; H\geq0.
\]
Since $W$ is a fixed point, we have that 
$
\mathcal T^H W=W$,
and therefore we deduce that 
$W\leq V_H$
for all $H$.
Taking the infimum over $H$ yields
\[
W\leq
\inf_{H\geq0}V_H
=
V^\star.
\]
Thus every bounded upper semicontinuous fixed point of
$\mathcal T$ lies below $V^\star$. Since
$V^\star\in\mathrm{USC}_b(\cX)$ and
$\mathcal T V^\star=V^\star$, $V^\star$ itself belongs to this
class. Hence, pointwise, we have that 
\[
V^\star
=
\sup\left\{
W\in\mathrm{USC}_b(\cX):
\mathcal T W=W
\right\}.
\]

\paragraph{Step 7: $\cC^\star(\cXsafe)$ is controlled invariant.}
Define
\[
\cC^\star(\cXsafe)
:=
\{x\in\cX:V^\star(x)\geq0\}.
\]
Since
\[
V^\star(x)
=
(\mathcal T V^\star)(x)
\leq
\marginfn(x),
\]
every $x\in\cC^\star(\cXsafe)$ satisfies
$\marginfn(x)\geq0$. Hence
$
\cC^\star(\cXsafe)\subseteq\cXsafe$.

Now fix $x\in\cC^\star(\cXsafe)$. The Bellman equation gives
\[
V^\star(x)
=
\min\left\{
\marginfn(x),
\max_{u\in\cU}V^\star\bigl(f(x,u)\bigr)
\right\}
\geq0.
\]
It follows that
\[
\max_{u\in\cU}
V^\star\bigl(f(x,u)\bigr)
\geq0.
\]
Because $V^\star\circ f(x,\cdot)$ is upper semicontinuous on the
compact set $\cU$, the maximum is attained. Thus there exists
$u^\star\in\cU$ such that
\[
V^\star\bigl(f(x,u^\star)\bigr)\geq0.
\]
Therefore
$
f(x,u^\star)\in\cC^\star(\cXsafe)$,
which proves that $\cC^\star(\cXsafe)$ is controlled invariant.

\paragraph{Step 8: $\cC^\star(\cXsafe)$ is maximal.}
Let $\cC\subseteq\cXsafe$ be any controlled-invariant set. We
show by induction that
$
V_H(x)\geq0$
for all $x\in\cC$ and all $H\geq0$.
For $H=0$, we have
$
V_0(x)=\marginfn(x)\geq0$
since $\cC\subseteq\cXsafe$.

Suppose that $V_H(x)\geq0$ for every $x\in\cC$. Fix
$x\in\cC$. Since $\cC$ is controlled invariant, there exists
$u_x\in\cU$ such that
$
f(x,u_x)\in\cC$.
By the induction hypothesis, we have that 
$
V_H\bigl(f(x,u_x)\bigr)\geq0$,
and hence
\[
\max_{u\in\cU}V_H\bigl(f(x,u)\bigr)
\geq
V_H\bigl(f(x,u_x)\bigr)
\geq0.
\]
Also $\marginfn(x)\geq0$, since $x\in\cC\subseteq\cXsafe$.
Therefore
\begin{align*}
V_{H+1}(x)
&=
\min\left\{
\marginfn(x),
\max_{u\in\cU}
V_H\bigl(f(x,u)\bigr)
\right\}\\
&\geq0.
\end{align*}
Thus the induction holds for every $H$.

Taking $H\to\infty$ gives
$V^\star(x)\geq0$
for all $x\in\cC$,
so that
$
\cC\subseteq\cC^\star(\cXsafe)$.
Since $\cC$ was an arbitrary controlled-invariant subset of
$\cXsafe$, $\cC^\star(\cXsafe)$ is the maximal
controlled-invariant subset of $\cXsafe$. This completes the proof.
\end{proof}
\section{Proof of Risk-Guided Replay Results}\label{app:risk_proofs}
We prove the variational optimality, concentration, and reference-distribution
limit in Theorem~\ref{thm:boundary-tilting}.

\begin{proof}[Proof of Theorem~\ref{thm:boundary-tilting}]
Since $b_\phi$ is bounded and measurable and $\mu$ is a probability
measure, the normalizing constant
\[
Z_\beta:=\mathbb E_{X\sim\mu}[e^{\beta b_\phi(X)}]
\]
is finite and strictly positive for every finite $\beta\geq0$.
Thus $d\rho_\beta/d\mu=e^{\beta b_\phi}/Z_\beta$ defines a
probability measure with a strictly positive density. In particular,
$\rho_\beta$ and $\mu$ are mutually absolutely continuous.

\paragraph{Part a.~\textbf{Variational optimality.}}
Fix $\beta>0$ and a probability measure $\rho\ll\mu$ with
$\KL(\rho|\mu)<\infty$. The density ratio satisfies
\[
\log\frac{d\rho}{d\rho_\beta}
=
\log\frac{d\rho}{d\mu}-\beta b_\phi+\log Z_\beta
\qquad \rho\text{-almost surely}.
\]
Integrating with respect to $\rho$ yields
\begin{equation}
\label{eq:boundary-kl-decomposition}
\KL(\rho|\rho_\beta)
=
\KL(\rho|\mu)-\beta\mathbb E_{X\sim\rho}[b_\phi(X)]
+\log Z_\beta.
\end{equation}
Thus we have the variational identity
\begin{equation}
\label{eq:boundary-variational-identity}
\mathbb E_{X\sim\rho}[b_\phi(X)]
-\frac{1}{\beta}\KL(\rho|\mu)
=
\frac{1}{\beta}\log Z_\beta
-\frac{1}{\beta}\KL(\rho|\rho_\beta).
\end{equation}
Relative entropy is nonnegative and vanishes only when its two
probability measures coincide. Indeed, writing $h=d\rho/d\rho_\beta$,
\[
\KL(\rho|\rho_\beta)
=
\int(h\log h-h+1)\,d\rho_\beta\geq0,
\]
with equality only if $h=1$ $\rho_\beta$-almost everywhere; here
$0\log0=0$.
The logarithm of $d\rho_\beta/d\mu$ is bounded, so
$\KL(\rho_\beta|\mu)<\infty$. Substituting $\rho=\rho_\beta$
into \eqref{eq:boundary-variational-identity} therefore attains
the upper bound $\beta^{-1}\log Z_\beta$, uniquely.
If instead $\KL(\rho|\mu)=\infty$, boundedness of $b_\phi$
makes the objective equal to $-\infty$, so such a measure cannot
maximize it. This proves the optimizer and optimal-value claims
over all admissible probability measures.

\paragraph{Part b.~\textbf{Concentration.}}
Fix $0<\delta<\varepsilon$, and define
\[
A_\varepsilon
:=
\left\{
x:
b_\phi(x)\leq b_\phi^\star-\varepsilon
\right\},
\qquad
B_\delta
:=
\left\{
x:
b_\phi(x)\geq b_\phi^\star-\delta
\right\}.
\]
We first bound the $\rho_\beta$-probability of
$A_\varepsilon$. By the definition of $\rho_\beta$,
\begin{align}
\rho_\beta(A_\varepsilon)
&=
\int_{A_\varepsilon}
\frac{e^{\beta b_\phi(x)}}{Z_\beta}
\,d\mu(x) =
\frac{
    \int_{A_\varepsilon}
    e^{\beta b_\phi(x)}
    \,d\mu(x)
}{
    Z_\beta
}.
\label{eq:tilting-prob-ratio}
\end{align}
For every $x\in A_\varepsilon$, we have
$
b_\phi(x)
\leq
b_\phi^\star-\varepsilon$,
and therefore, since $\beta\geq0$, we have
$
e^{\beta b_\phi(x)}
\leq
e^{\beta(b_\phi^\star-\varepsilon)}$.
Consequently, the bounds hold:
\begin{align}
\int_{A_\varepsilon}
e^{\beta b_\phi(x)}
\,d\mu(x)
&\leq
e^{\beta(b_\phi^\star-\varepsilon)}
\mu(A_\varepsilon) \leq
e^{\beta(b_\phi^\star-\varepsilon)}.
\label{eq:tilting-numerator}
\end{align}

We next lower bound the normalizing constant. Since
$B_\delta$ is a measurable subset of the state space,
\begin{align}
Z_\beta
&=
\int
e^{\beta b_\phi(x)}
\,d\mu(x)\geq
\int_{B_\delta}
e^{\beta b_\phi(x)}
\,d\mu(x).
\label{eq:tilting-Z-restrict}
\end{align}
For every $x\in B_\delta$, we have
$
b_\phi(x)
\geq
b_\phi^\star-\delta$,
so that
$
e^{\beta b_\phi(x)}
\geq
e^{\beta(b_\phi^\star-\delta)}$.
Thus
\begin{align}
Z_\beta
\geq
\int_{B_\delta}
e^{\beta(b_\phi^\star-\delta)}
\,d\mu(x) =
\mu(B_\delta)
e^{\beta(b_\phi^\star-\delta)}.
\label{eq:tilting-denominator}
\end{align}

Moreover, we have that 
$
\mu(B_\delta)>0.
$
Indeed, by definition we have that 
$
b_\phi^\star
=
\operatorname*{ess\,sup}_{X\sim\mu}
b_\phi(X)$.
If $\mu(B_\delta)=0$ for some $\delta>0$, then
$b_\phi(x)<b_\phi^\star-\delta$ for $\mu$-almost every $x$,
which would imply
\[
\operatorname*{ess\,sup}_{X\sim\mu}b_\phi(X)
\leq
b_\phi^\star-\delta
<
b_\phi^\star,
\]
a contradiction. Hence the denominator in the desired bound is
strictly positive.

Substituting \eqref{eq:tilting-numerator} and
\eqref{eq:tilting-denominator} into
\eqref{eq:tilting-prob-ratio} yields
\begin{align}
\rho_\beta(A_\varepsilon)
&\leq
\frac{
e^{\beta(b_\phi^\star-\varepsilon)}
}{
\mu(B_\delta)
e^{\beta(b_\phi^\star-\delta)}
}
=
\frac{
e^{-\beta(\varepsilon-\delta)}
}{
\mu(B_\delta)
}.
\end{align}
Recalling the definitions of $A_\varepsilon$ and $B_\delta$,
we obtain
\[
\rho_\beta
\left(
b_\phi\leq b_\phi^\star-\varepsilon
\right)
\leq
\frac{
e^{-\beta(\varepsilon-\delta)}
}{
\mu
\left(
b_\phi\geq b_\phi^\star-\delta
\right)
}.
\]

Since $\varepsilon-\delta>0$ and
$\mu(B_\delta)>0$ does not depend on $\beta$, the right-hand
side converges to zero exponentially fast as
$\beta\to\infty$. Therefore, for every fixed
$\varepsilon>0$, the tilted distribution assigns vanishing
probability to states whose boundary utility lies at least
$\varepsilon$ below its essential supremum. In this sense,
$\rho_\beta$ concentrates on states with near-maximal boundary
utility as $\beta$ increases.

\paragraph{Part c.~Support and limit at zero.}
It remains to establish the support statement, and the limit at zero. For every finite
$\beta$,
\[
\frac{d\rho_\beta}{d\mu}(x)
=
\frac{e^{\beta b_\phi(x)}}{Z_\beta}
>0
\qquad
\text{$\mu$-almost everywhere}.
\]
Therefore $\rho_\beta\ll\mu$. Conversely, let $A$ be any
measurable set satisfying $\rho_\beta(A)=0$. Then
\[
0
=
\rho_\beta(A)
=
\int_A
\frac{e^{\beta b_\phi(x)}}{Z_\beta}
\,d\mu(x).
\]
The integrand is strictly positive $\mu$-almost everywhere, so
the integral can vanish only if $\mu(A)=0$. Hence
$\mu\ll\rho_\beta$. Thus $\rho_\beta$ and $\mu$ are mutually
absolutely continuous and have exactly the same null sets. In
particular, they have the same support for every finite
$\beta$.

Let us now show the limit at zero. 
Consider $\beta=0$ so that  $Z_0=1$ and $d\rho_0/d\mu=1$, which in turn implies $\rho_0=\mu$.
Choose $M<\infty$ such that $|b_\phi(x)|\leq M$. Then
\[
e^{-\beta M}\leq Z_\beta\leq e^{\beta M},
\qquad
e^{-2\beta M}\leq\frac{d\rho_\beta}{d\mu}(x)\leq e^{2\beta M}.
\]
It follows that
\[
\|\rho_\beta-\mu\|_{\mathrm{TV}}
=
\frac12\int\left|\frac{d\rho_\beta}{d\mu}-1\right|\,d\mu
\leq\frac12(e^{2\beta M}-1)\longrightarrow0
\qquad\text{as }\beta\downarrow0.
\]
This proves the reference-distribution limit and completes the proof.
\end{proof}

\section{High-Probability Safety Guarantees}
\label{app:hp-safety}
In this section we provide proofs of two high-probability safety guarantees from Section~\ref{sec:safety-guarantees}, and some additional implications for feasibility. 

\subsection{Feasibility of the Learned Teacher}
\label{app:learned-teacher-feasibility}
The preceding transfer result is conditional on feasibility of the
tightened teacher problem~\eqref{eq:qcbf-training-filter}. We next
characterize a sufficient condition under which such a teacher action
is guaranteed to exist. This makes explicit the tradeoff introduced by
the tightening margin: while a sufficiently large $\Delta$ absorbs
learning and distillation errors, feasibility additionally requires
that the underlying Q-CBF possess enough slack at the current state.

\begin{proposition}[Sufficient condition for learned-teacher feasibility]
\label{prop:learned-teacher-feasibility}
Fix $x\in\cC^\star(\cXsafe)$ and suppose that $\classkl$ is
$L_\classkl$-Lipschitz continuous on the relevant value range. Suppose
\[
|V_\phi(x)-V^\star(x)|
\leq
\varepsilon_V,
\qquad
|Q_\phi(x,r)-Q^\star(x,r)|
\leq
\varepsilon_Q
\quad
\forall r\in\cR.
\]
For the fixed tightening margin $\Delta\geq0$ used in
\eqref{eq:qcbf-training-filter}, if
\begin{equation}
\label{eq:teacher-feasibility-slack}
V^\star(x)
-
\classkl\bigl(V^\star(x)\bigr)
\geq
\Delta
+
\varepsilon_Q
+
L_\classkl\varepsilon_V,
\end{equation}
then
$
\mathcal A_\Delta(x)\neq\emptyset$.
\end{proposition}

\begin{proof}
Let
$r^\star
\in
\argmax_{r\in\cR}Q^\star(x,r)$.
Since
$
V^\star(x)
=
\max_{r\in\cR}Q^\star(x,r)$,
we have
$Q^\star(x,r^\star)=V^\star(x)$.
The critic approximation bound therefore implies
$Q_\phi(x,r^\star)
\geq
V^\star(x)-\varepsilon_Q$.
Moreover, Lipschitz continuity of $\classkl$ and the state-value
approximation bound give
$
\classkl\bigl(V_\phi(x)\bigr)
\leq
\classkl\bigl(V^\star(x)\bigr)
+
L_\classkl\varepsilon_V$.
Hence
\begin{align}
Q_\phi(x,r^\star)
-
\classkl\bigl(V_\phi(x)\bigr)
&\geq
V^\star(x)
-
\classkl\bigl(V^\star(x)\bigr)
-
\varepsilon_Q
-
L_\classkl\varepsilon_V
\geq
\Delta,
\end{align}
where the final inequality follows from
\eqref{eq:teacher-feasibility-slack}. Therefore
$
Q_\phi(x,r^\star)
\geq
\classkl\bigl(V_\phi(x)\bigr)+\Delta$,
so that $r^\star\in\mathcal A_\Delta(x)$ and the learned teacher problem is
feasible.
\end{proof}

The proposition identifies the amount of exact Q-CBF slack required to
support the fixed tightening used by the learned teacher. In
particular, define
\begin{equation}
\label{eq:teacher-feasible-interior}
\cC^\star_{\rm feas}(\Delta)
:=
\left\{
x\in\cC^\star(\cXsafe):
V^\star(x)
-
\classkl\bigl(V^\star(x)\bigr)
\geq
\Delta
+
\varepsilon_Q
+
L_\classkl\varepsilon_V
\right\}.
\end{equation}
On the stated approximation event, the tightened learned teacher is
guaranteed to be feasible for every
$x\in\cC^\star_{\rm feas}(\Delta)$. This set characterizes a sufficient
teacher-feasibility region; it is not asserted to be invariant.
Together with
Corollary~\ref{cor:learned-qcbf-teacher-transfer}, this exposes the two
roles of $\Delta$: it must be sufficiently large to absorb learning and
distillation errors, while remaining small enough relative to the
available Q-CBF slack for the teacher constraint to remain feasible.

\subsection{Proof of Theorem~\ref{thm:main-stage-one-safety}}
Let $\mathcal{F}_t$ denote the information available immediately before
sampling the residual at time $t$, including the realized training
outcome, the learned functions $Q_\phi$ and $V_\phi$, and the
closed-loop history up to the current state $x_t$. Since the assumed
per-state bound holds for almost every realized training outcome and
$R_t$ is sampled from the stated conditional policy, whenever
$x_t\in\cC^\star(\cXsafe)$ we have, almost surely,
\[
\Pr \left(
    Q_\phi(x_t,R_t)
    <
    \classkl\bigl(V_\phi(x_t)\bigr)+\Delta
    \,\middle|\,\mathcal F_t
\right)
\leq\delta,
\]
where $R_t\sim\pi^\Ir(\cdot\mid x_t)$.

For $t=0,\ldots,H-1$, define the action-selection failure event
\[
\mathcal E_t
:=
\left\{
x_t\in\cC^\star(\cXsafe),\;
Q_\phi(x_t,R_t)
<
\classkl\bigl(V_\phi(x_t)\bigr)+\Delta
\right\}.
\]
By the assumed uniform per-state bound,
\begin{align}
\Pr(\mathcal E_t)
&=
\E \left[
    \Pr(\mathcal E_t\mid\mathcal F_t)
\right] =
\E \left[
    \mathbf 1_{\{x_t\in\cC^\star(\cXsafe)\}}
    \Pr \left(
        Q_\phi(x_t,R_t)
        <
        \classkl\bigl(V_\phi(x_t)\bigr)+\Delta
        \,\middle|\,\mathcal F_t
    \right)
\right]\\
& \leq
\E \left[
    \mathbf 1_{\{x_t\in\cC^\star(\cXsafe)\}}\delta
\right]
\leq\delta.
\label{eq:stage-one-failure-prob}
\end{align}
No independence between the residual samples at different time steps
is required.

We next establish the deterministic one-step safety implication on the
good certificate event $\mathcal G_Q$.  Fix any
$x\in\cC^\star(\cXsafe)$ and residual $r\in\cR$ satisfying
\begin{equation}
\label{eq:stage-one-success-condition-proof}
Q_\phi(x,r)
\geq
\classkl\bigl(V_\phi(x)\bigr)+\Delta.
\end{equation}
On $\mathcal G_Q$, the uniform critic approximation bound gives us that 
$
Q^\star(x,r)
\geq
Q_\phi(x,r)-\varepsilon_Q$.
Combining this with
\eqref{eq:stage-one-success-condition-proof} yields
\begin{equation}
Q^\star(x,r)
\geq
\classkl\bigl(V_\phi(x)\bigr)
+\Delta-\varepsilon_Q.
\label{eq:stage-one-q-first}
\end{equation}
Moreover, on $\mathcal G_Q$, we have that 
$\left|V_\phi(x)-V^\star(x)\right|
\leq\varepsilon_V$.
Since $\classkl$ is $L_\classkl$-Lipschitz on the relevant value
range,
\begin{align}
\classkl\bigl(V_\phi(x)\bigr)
&\geq
\classkl\bigl(V^\star(x)\bigr)
-
L_\classkl
\left|V_\phi(x)-V^\star(x)\right| \geq
\classkl\bigl(V^\star(x)\bigr)
-
L_\classkl\varepsilon_V.
\label{eq:stage-one-alpha-bound}
\end{align}
Substituting \eqref{eq:stage-one-alpha-bound} into
\eqref{eq:stage-one-q-first} gives us that 
\begin{align}
Q^\star(x,r)
&\geq
\classkl\bigl(V^\star(x)\bigr)
+
\Delta
-
\varepsilon_Q
-
L_\classkl\varepsilon_V \nonumber\\
&\geq
\classkl\bigl(V^\star(x)\bigr),
\label{eq:stage-one-exact-admissible}
\end{align}
where the last inequality follows from
$
\Delta\geq
\varepsilon_Q+L_\classkl\varepsilon_V.
$
Thus every residual satisfying the tightened learned Q-CBF condition
on $\mathcal G_Q$ belongs to the exact Q-CBF admissible set
$\cR^\star_{\mathrm{QCBF}}(x)$.

We now show that such an action preserves
$\cC^\star(\cXsafe)$.  By definition of the exact state--action
safety value,
\[
Q^\star(x,r)
=
\min\left\{
    g(x),\,
    V^\star\bigl(f_{\rm res}(x,r)\bigr)
\right\}.
\]
Let $x^+=f_{\rm res}(x,r)$.  Since
$x\in\cC^\star(\cXsafe)$, we have $V^\star(x)\geq0$, and since
$\classkl$ is class-$\cK$ on $[0,\infty)$,
$\classkl(V^\star(x))\geq0$.  From
\eqref{eq:stage-one-exact-admissible}, we have that
$
\min\left\{
    g(x),V^\star(x^+)
\right\}
\geq
\classkl\bigl(V^\star(x)\bigr)$.
In particular, we have the bounds
$
V^\star(x^+)
\geq
\classkl\bigl(V^\star(x)\bigr)
\geq0$.
Hence $x^+\in\cC^\star(\cXsafe)$.  Therefore, on $\mathcal G_Q$,
any time $x_t\in\cC^\star(\cXsafe)$ and
$\mathcal E_t$ does not occur, we have
$x_{t+1}\in\cC^\star(\cXsafe)$.

Consider now the event
\[
\mathcal A_H
:=
\mathcal G_Q
\cap
\bigcap_{t=0}^{H-1}\mathcal E_t^c.
\]
We claim that on $\mathcal A_H$,
\[
x_t\in\cC^\star(\cXsafe)
\qquad
\text{for every }t=0,\ldots,H.
\]
This follows by induction.  The claim holds for $t=0$ by assumption.
Suppose it holds at time $t<H$.  Then
$x_t\in\cC^\star(\cXsafe)$.  Since $\mathcal E_t^c$ occurs, the
sampled residual satisfies
$
Q_\phi(x_t,R_t)
\geq
\classkl\bigl(V_\phi(x_t)\bigr)+\Delta$.
Since $\mathcal G_Q$ also occurs, the one-step argument above gives
$x_{t+1}\in\cC^\star(\cXsafe)$.  This completes the induction.

Consequently, we have that 
\[
\left\{
x_t\in\cC^\star(\cXsafe)
\ \forall\,t=0,\ldots,H
\right\}^c
\subseteq
\mathcal G_Q^c
\cup
\bigcup_{t=0}^{H-1}\mathcal E_t.
\]
Applying the union bound, using
$\Pr(\mathcal G_Q^c)\leq\delta_Q$ and
\eqref{eq:stage-one-failure-prob}, yields
\begin{align*}
\Pr \left(
    \exists\,t\leq H:
    x_t\notin\cC^\star(\cXsafe)
\right)
&\leq
\Pr(\mathcal G_Q^c)
+
\sum_{t=0}^{H-1}\Pr(\mathcal E_t)\leq
\delta_Q+H\delta.
\end{align*}
Taking complements proves
$
\Pr \left(
    x_t\in\cC^\star(\cXsafe)
    \ \forall\,t=0,\ldots,H
\right)
\geq
1-\delta_Q-H\delta$.
Finally,
$\cC^\star(\cXsafe)\subseteq\cXsafe$, while
$\cXfail=\{x:g(x)<0\}$ is disjoint from $\cXsafe$.  Therefore
\[
\left\{
\exists\,t\leq H:x_t\in\cXfail
\right\}
\subseteq
\left\{
\exists\,t\leq H:
x_t\notin\cC^\star(\cXsafe)
\right\},
\]
and hence
$
\Pr \left(
    \exists\,t\leq H:x_t\in\cXfail
\right)
\leq
\min\{1,\delta_Q+H\delta\}$.

If a deterministic Q-CBF filter is feasible and enforces the
tightened learned condition at every
$x\in\cC^\star(\cXsafe)$, then the events
$\mathcal E_t$ are empty.  Equivalently, $\delta=0$, and the bound
reduces to $1-\delta_Q$, which concludes the proof.

\subsection{Proof of Theorem~\ref{thm:main-stage-two-safety}}
Let $\mathcal F_t$ again denote the information available immediately
before sampling the Stage~II residual at time $t$, including the
realized training outcome, the learned functions $Q_\phi$ and
$V_\phi$, and the trajectory history up to $x_t$.  We interpret the
assumed distillation bound conditionally on this information; thus,
whenever $x_t\in\cC^\star(\cXsafe)$,
\[
\Pr \left(
    \left\|
    R_t-
    \rteach_{\phi,\Delta}(x_t,R_t)
    \right\|
    >
    \varepsilon_\pi
    \,\middle|\,\mathcal F_t
\right)
\leq\delta_\pi,
\]
where
$R_t\sim\pitask_\theta(\cdot\mid x_t)$ is the residual actually
executed by the filter-free Stage~II policy.

For each $t=0,\ldots,H-1$, define the distillation-failure event
\[
\mathcal E_t^\pi
:=
\left\{
x_t\in\cC^\star(\cXsafe),\;
\left\|
R_t-
\rteach_{\phi,\Delta}(x_t,R_t)
\right\|
>
\varepsilon_\pi
\right\}.
\]
The assumed uniform per-state bound implies
\begin{align}
\Pr(\mathcal E_t^\pi)
&=
\E \left[
    \Pr(\mathcal E_t^\pi\mid\mathcal F_t)
\right]\leq
\E \left[
    \mathbf 1_{\{x_t\in\cC^\star(\cXsafe)\}}
    \delta_\pi
\right]
\leq\delta_\pi.
\label{eq:stage-two-failure-prob}
\end{align}
Again, no independence assumption across time is required.

We first establish the one-step implication on the certificate-good
event $\mathcal G_Q$.  Fix
$x\in\cC^\star(\cXsafe)$ and let
$r\in\cR$ be a task-policy proposal satisfying
\[
\left\|
r-\rteach_{\phi,\Delta}(x,r)
\right\|
\leq\varepsilon_\pi.
\]
For brevity, write
$
\bar r
:=
\rteach_{\phi,\Delta}(x,r)$.
By feasibility of the teacher optimization problem and the constraint
in \eqref{eq:main-stage-two-teacher}, we have that
$c_{\phi,\Delta}(x,\bar r)\leq0$.
Using
\[
c_{\phi,\Delta}(x,\bar r)
=
\classkl\bigl(V_\phi(x)\bigr)
+\Delta
-
Q_\phi(x,\bar r),
\]
this is equivalent to
\begin{equation}
Q_\phi(x,\bar r)
\geq
\classkl\bigl(V_\phi(x)\bigr)+\Delta.
\label{eq:stage-two-teacher-feasible}
\end{equation}
On $\mathcal G_Q$, the critic approximation bound gives
\begin{equation}
Q^\star(x,\bar r)
\geq
Q_\phi(x,\bar r)-\varepsilon_Q.
\label{eq:stage-two-critic-transfer}
\end{equation}
Since $Q^\star(x,\cdot)$ is $L_Q$-Lipschitz continuous, we have that 
\begin{align}
Q^\star(x,r)
&\geq
Q^\star(x,\bar r)
-
L_Q\|r-\bar r\| \geq
Q^\star(x,\bar r)
-
L_Q\varepsilon_\pi.
\label{eq:stage-two-policy-transfer}
\end{align}
Combining
\eqref{eq:stage-two-teacher-feasible}--%
\eqref{eq:stage-two-policy-transfer} yields
\begin{align}
Q^\star(x,r)
&\geq
Q_\phi(x,\bar r)
-
\varepsilon_Q
-
L_Q\varepsilon_\pi \geq
\classkl\bigl(V_\phi(x)\bigr)
+
\Delta
-
\varepsilon_Q
-
L_Q\varepsilon_\pi.
\label{eq:stage-two-before-v-error}
\end{align}
On $\mathcal G_Q$, we have that 
$
|V_\phi(x)-V^\star(x)|
\leq\varepsilon_V$.
Therefore, by the $L_\classkl$-Lipschitz continuity of $\classkl$, we have 
\begin{align}
\classkl\bigl(V_\phi(x)\bigr)
&\geq
\classkl\bigl(V^\star(x)\bigr)
-
L_\classkl
|V_\phi(x)-V^\star(x)| \nonumber\\
&\geq
\classkl\bigl(V^\star(x)\bigr)
-
L_\classkl\varepsilon_V.
\label{eq:stage-two-value-transfer}
\end{align}
Substituting \eqref{eq:stage-two-value-transfer} into
\eqref{eq:stage-two-before-v-error} yields
\begin{align}
Q^\star(x,r)
&\geq
\classkl\bigl(V^\star(x)\bigr)
+
\Delta
-
\varepsilon_Q
-
L_\classkl\varepsilon_V
-
L_Q\varepsilon_\pi \geq
\classkl\bigl(V^\star(x)\bigr),
\label{eq:stage-two-exact-admissible}
\end{align}
where the final inequality follows from the assumed tightening
\[
\Delta
\geq
\varepsilon_Q+
L_\classkl\varepsilon_V+
L_Q\varepsilon_\pi.
\]
Thus the filter-free task action itself satisfies the exact Q-CBF
condition---namely, 
$
r\in\cR^\star_{\mathrm{QCBF}}(x)$.

We next verify that this exact admissibility condition preserves the
maximal controlled-invariant safe set.  Let
$
x^+=f_{\rm res}(x,r)$.
By definition,
$
Q^\star(x,r)
=
\min\left\{
g(x),V^\star(x^+)
\right\}$.
Since $x\in\cC^\star(\cXsafe)$,
$V^\star(x)\geq0$.  Hence
$\classkl(V^\star(x))\geq0$, and
\eqref{eq:stage-two-exact-admissible} implies
\[
V^\star(x^+)
\geq
Q^\star(x,r)
\geq
\classkl\bigl(V^\star(x)\bigr)
\geq0.
\]
Therefore
\[
x^+\in\cC^\star(\cXsafe).
\]
We have consequently shown that, on $\mathcal G_Q$, whenever
$x_t\in\cC^\star(\cXsafe)$ and $\mathcal E_t^\pi$ does not occur,
the actually executed Stage~II action $R_t$ keeps the next state in
$\cC^\star(\cXsafe)$.

Now define
\[
\mathcal A_H^\pi
:=
\mathcal G_Q
\cap
\bigcap_{t=0}^{H-1}
(\mathcal E_t^\pi)^c.
\]
On $\mathcal A_H^\pi$, we prove by induction that
\[
x_t\in\cC^\star(\cXsafe)
\qquad
\forall\,t=0,\ldots,H.
\]
The base case holds because
$x_0\in\cC^\star(\cXsafe)$ by assumption.  Suppose the claim holds
at time $t<H$.  Then $x_t\in\cC^\star(\cXsafe)$, and because
$(\mathcal E_t^\pi)^c$ occurs, we have the bound
\[
\left\|
R_t-
\rteach_{\phi,\Delta}(x_t,R_t)
\right\|
\leq\varepsilon_\pi.
\]
Together with $\mathcal G_Q$, the one-step argument above implies
$x_{t+1}\in\cC^\star(\cXsafe)$.  This proves the induction claim.

It follows that
\[
\left\{
x_t\in\cC^\star(\cXsafe)
\ \forall\,t=0,\ldots,H
\right\}^c
\subseteq
\mathcal G_Q^c
\cup
\bigcup_{t=0}^{H-1}
\mathcal E_t^\pi.
\]
Using the union bound,
$\Pr(\mathcal G_Q^c)\leq\delta_Q$, and
\eqref{eq:stage-two-failure-prob}, we have that 
\begin{align*}
\Pr \left(
    \exists\,t\leq H:
    x_t\notin\cC^\star(\cXsafe)
\right)
&\leq
\Pr(\mathcal G_Q^c)
+
\sum_{t=0}^{H-1}
\Pr(\mathcal E_t^\pi)\leq
\delta_Q+H\delta_\pi.
\end{align*}
Taking complements proves
\[
\Pr \left(
    x_t\in\cC^\star(\cXsafe)
    \ \forall\,t=0,\ldots,H
\right)
\geq
1-\delta_Q-H\delta_\pi.
\]
Finally, since
$\cC^\star(\cXsafe)\subseteq\cXsafe$ and
$\cXfail\cap\cXsafe=\varnothing$,
\[
\left\{
\exists\,t\leq H:x_t\in\cXfail
\right\}
\subseteq
\left\{
\exists\,t\leq H:
x_t\notin\cC^\star(\cXsafe)
\right\}.
\]
Therefore
\[
\Pr \left(
    \exists\,t\leq H:x_t\in\cXfail
\right)
\leq
\min\{1,\delta_Q+H\delta_\pi\},
\]
which completes the proof.

\section{Additional Experimental Details}
\label{app:additional-experiments}

This appendix provides implementation, training, and evaluation details for the experiments in Section~\ref{sec:experiments}.

\subsection{Platform and Observations}

\paragraph{Platform and control.} We use a 29-DoF Unitree G1 in \texttt{mjlab} with MuJoCo, training in $4096$ parallel environments at $50$~Hz. The policy produces joint-position residuals around a frozen standing or walking controller, following~\eqref{eq:whole-body-dynamics}. Residuals are clipped to $[-1,1]^{29}$ before execution. The action scale is $S_{jj}=\tau_j^{\max}/k_{p,j}$, where $\tau_j^{\max}$ and $k_{p,j}$ are the nominal effort limit and proportional gain. The base controller's output is expressed in the same action units.

\paragraph{Observations.} Proprioception includes base angular velocity, projected gravity, relative joint positions, joint velocities, and the previous action. Dodgeball additionally uses four ball-position observations spanning $60$~ms and ball positions relative to nine body points: the torso and both elbows, wrists, knees, and ankles. Positions are expressed in the base yaw frame; the four-frame history is relative to the pelvis. There is no explicit ball-velocity input. Low-obstacle observations describe the obstacle's relative planar position, absolute lower and upper heights, passage half-width, and clearance from each robot collision geometry. Observations also include the velocity command and obstacle-cycle phase. The hardware dodgeball interface holds the phase input constant.

\subsection{Safety Specification}

We combine collision clearance, pelvis orientation, and trunk height. Let $d_{\mathrm{collision}}(x)$ denote the minimum obstacle clearance across the robot's collision geometries, $\theta(x)$ the pelvis tilt from vertical, and $z_{\mathrm{low}}(x)$ the height of the lowest trunk collision geometry. The specification and its normalized margins are
\begin{align*}
g(x) &= \min\{h_{\mathrm{collision}}(x),h_{\mathrm{tilt}}(x),h_{\mathrm{ground}}(x)\},\\
h_{\mathrm{collision}}(x) &= \min\left\{1,\frac{d_{\mathrm{collision}}(x)}{d_{\mathrm{ref}}}\right\},\\
h_{\mathrm{tilt}}(x) &= \frac{\cos\theta(x)-\cos\theta_{\max}}{1-\cos\theta_{\max}},\\
h_{\mathrm{ground}}(x) &= \min\left\{1,\frac{z_{\mathrm{low}}(x)-z_{\mathrm{fail}}}{z_{\mathrm{ref}}-z_{\mathrm{fail}}}\right\}.
\end{align*}
We use $\theta_{\max}=0.80$~rad, $z_{\mathrm{fail}}=0.15$~m, and $z_{\mathrm{ref}}=0.30$~m. Thus, the ground margin is zero at the failure height and reaches one at the reference height. The collision scale is $d_{\mathrm{ref}}=0.30$~m for dodgeball and $0.15$~m for low obstacles; these scales normalize clearance without changing the zero-clearance collision boundary.

Collision clearance is evaluated across all $33$ robot collision geometries. Dodgeball uses ball-to-body distances with the ball and body radius parameters subtracted. Low-obstacle clearance is estimated by sampling along the robot's collision capsules and evaluating their signed distances to the obstacle panels. The trunk-height margin considers the pelvis, torso, head, and hips. Self-collision and lane departure are addressed separately through rewards and termination conditions.

\subsection{Task Setups}

\paragraph{Dodgeball.} The ball has radius $0.11$~m and mass $0.20$~kg. Launches span a $\pm25^\circ$ frontal cone, with horizontal distance sampled uniformly from $[2.0,3.5]$~m, launch height from $[0.8,2.0]$~m, and flight time from $[0.45,0.80]$~s. Three aim regions are sampled equally: the head, with height $[1.10,1.28]$~m and lateral offset $[-0.06,0.06]$~m, and the left and right upper body, with height $[0.80,0.98]$~m and lateral-offset magnitude $[0.15,0.24]$~m. We add uniform jitter of $\pm0.05$~m to aim height and lateral offset. For $15\%$ of throws, the lateral offset is replaced by a randomly signed magnitude in $[0.45,0.80]$~m to generate misses. The aim point includes constant-velocity lead, and launch velocity is computed ballistically for the sampled flight time. Throws occur every $3$~s, with a $1.2$~s active window. Stage-I episodes last $15$~s and Stage-II episodes last $30$~s.

\paragraph{Locomotion beneath obstacles.} The robot is commanded to walk forward at $0.5$~m/s, with zero lateral and yaw-rate commands. Obstacles begin $2.5$~m ahead and move toward the robot at $1.1$~m/s. Wall heights are sampled uniformly from $\{1.25,1.15,1.05,0.95\}$~m, with passage half-width $7.5$~m and obstacle thickness $0.06$~m. Obstacle panels are moved kinematically, so contact does not displace them. A new obstacle appears every $5$~s, giving six scheduled encounters per $30$~s episode. Obstacle speed and spawn distance are fixed.

\paragraph{Resets and termination.} Default resets use the nominal standing pose and zero velocity. Both tasks include $20\%$ obstacle-free episodes during Stage II. Episodes terminate at the time limit or when pelvis tilt exceeds $70^\circ$. Stage II also terminates on obstacle collision; Stage I retains collision transitions for safety learning. Dodgeball uses a $3$~m radial station boundary, while locomotion uses a $3$~m lateral lane boundary.

\subsection{Training Details}

\paragraph{Stage I.} We learn the safety policy and critic ensemble using the settings in Table~\ref{tab:appendix-training}. The safety policy is a tanh-squashed Gaussian, and the ensemble is reduced by its pointwise minimum. Ordinary task configurations use two critics. The default safety discount factor increases linearly from $0.90$ to $0.99$. The dodgeball configuration uses $25000$ training iterations and a $4000$ iteration discount schedule.

\begin{table}[!htbp]
\centering
\small
\caption{Key training hyperparameters. All networks use ELU activations.}
\label{tab:appendix-training}
\begin{tabular}{@{}ll@{}}
\toprule
Setting & Value \\
\midrule
\multicolumn{2}{@{}l}{\textit{Stage I: safety synthesis}} \\
Safety-policy hidden layers & $(256,256,128)$ \\
Safety-critic hidden layers & $(256,256,256)$ \\
Policy / critic learning rates & $10^{-4}$ / $3\times10^{-4}$ \\
Replay capacity / minibatch size & $500000$ / $512$ transitions \\
Collection / updates per iteration & $16$ steps per environment / $8$ updates \\
\midrule
\multicolumn{2}{@{}l}{\textit{Stage II: PPO}} \\
Policy and value-network hidden layers & $(512,256,128)$ \\
Learning rate & $10^{-3}$, adaptive \\
Rollout length & $24$ steps per environment \\
Updates per iteration & $5$ epochs, $4$ minibatches per epoch \\
Clipping parameter / discount / GAE & $0.2$ / $0.99$ / $0.95$ \\
Training iterations & $10000$ \\
\bottomrule
\end{tabular}
\end{table}

\paragraph{Replay sampling.} Margin replay uses weights proportional to $\exp(-|g(x)|/\tau)+p_{\min}$, with default values $\tau=0.30$ and $p_{\min}=0.20$. Value replay uses weights proportional to $\exp(-\beta|\widehat V(x)|)$, where $\widehat V(x)=\max_{j\leq K}Q_{\bar\phi}(x,r_j)$ is computed from $K=4$ safety-policy samples. Value scores are used without standardization, and the value sampler does not include an additive probability floor.

\paragraph{Concentration sweep.} The Stage-I walking sweep uses $\beta\in\{8,10,12,14,16,18,20,22,25\}$, together with margin-replay controls $(\tau,p_{\min})=(0.30,0.20)$ and $(0.22,0.05)$. Each setting uses seeds $0$, $1$, and $2$, $15000$ iterations, and ten critics. The environment, base controller, safety specification, and policy objective are fixed across the sweep. These experiments precede AMP-based task training. The deployed duck and limbo policies use separately synthesized two-critic barriers.

\paragraph{Stage-II teacher.} The safety critic and policy remain frozen. We use $\alpha(v)=\kappa v$ with $\kappa=0.99$ and $\Delta=0$. The teacher evaluates $16$ equally spaced candidates along the segment from the task proposal to the safety-policy action in composed joint-action coordinates, selecting the first feasible candidate when one exists. The correction provides counterfactual reward feedback, but the task policy's clipped residual is executed. We use $C=2$, $\sigma_r=2$, and $w_{\mathrm{Q}}=1$ in the Q-CBF reward.

\paragraph{Stage-II objective.} PPO combines Q-CBF feedback, an AMP motion prior, and control regularization for joint limits, self-collision, actuation, and motion smoothness, together with trunk-height and fall penalties. The walking task adds linear- and angular-velocity tracking rewards, each with weight $2$. No reward prescribes a particular crouching or leaning posture. Reward contributions are scaled by the control period as in~\eqref{eq:qcbf-guided-policy-objective}, and PPO uses observation normalization.

\subsection{Domain Randomization}

Table~\ref{tab:appendix-randomization} summarizes the shared parameter randomization and observation noise. Encoder biases and physical parameters are sampled once per environment at startup and held fixed; the foot-friction draw is shared across that environment's foot geometries. Additive observation noise is sampled during training. The simulation does not model action or observation delay.

\begin{table}[!htbp]
\centering
\small
\caption{Shared randomization settings. All listed ranges are uniform. Base linear-velocity noise applies only when that observation is included.}
\label{tab:appendix-randomization}
\begin{tabular}{@{}ll@{}}
\toprule
Quantity & Range \\
\midrule
\multicolumn{2}{@{}l}{\textit{Parameter randomization}} \\
Joint-encoder bias & $\pm0.015$~rad \\
Torso center-of-mass offset & $\pm0.025$~m in $x,y$; $\pm0.030$~m in $z$ \\
Foot friction, low-obstacle task & $[0.3,1.2]$ \\
\midrule
\multicolumn{2}{@{}l}{\textit{Additive observation noise}} \\
Base linear velocity & $\pm0.5$~m/s \\
Base angular velocity & $\pm0.2$~rad/s \\
Projected gravity & $\pm0.05$ \\
Joint position & $\pm0.01$~rad \\
Joint velocity & $\pm1.5$~rad/s \\
\bottomrule
\end{tabular}
\end{table}

\subsection{Evaluation and Deployment}

\paragraph{Simulation.} Policies are evaluated deterministically. Dodgeball outcomes are measured per throw, with any body contact counted as a hit and falls reported separately. Low-obstacle encounters are classified as cleared, struck, or aborted. The deployed duck and limbo comparison uses $2048$ encounters per policy. Rate intervals are $95\%$ Wilson intervals.

\paragraph{Hardware.} Deployment combines the task policy and frozen base through the residual interface, without the learned safety critic or Stage-I safety policy. Synthetic evaluation uses the same $59$ scripted trajectories per controller with real robot proprioception. Physical evaluation uses $26$ human throws per controller, tracked with motion capture. Synthetic trials use paired tests: exact McNemar for hits and Wilcoxon signed-rank for drift and peak tilt. Human throws use Fisher's exact test for hits and Mann--Whitney $U$ tests for drift and peak tilt.

\section{Algorithm Details}
In this appendix, we provide pseudocode for the Stage-I Q-CBF synthesis, the boundary-guided replay sampling procedure, and Stage-II task policy safety distillation.

\begin{algorithm}[h]
\small
\caption{Stage-I Q-CBF Synthesis}
\label{alg:appendix-stage-one}
\begin{algorithmic}[1]
\Require
frozen base policy $\pibase$,
safety margin $g$,
critic ensemble $\{Q_{\phi_j}\}_{j=1}^n$,
safety policy $\pisafe_\psi$,
replay buffer $\cD$,
replay rule $\mathsf{Replay}$,
discount schedule $\{\gamma_k\}$,
target-update rate $\tau$

\State Initialize $\{Q_{\phi_j}\}_{j=1}^n$, $\pisafe_\psi$, and $\cD$
\State Set $\bar\phi_j\gets\phi_j$ for $j=1,\ldots,n$

\For{Stage-I iterations $k=1,\ldots,K$}

    \For{environment interaction steps}\Comment{ \textit{Environment collection}}
        \State Observe state $x$
        and sample residual
        $r\sim\pisafe_\psi(\cdot\mid x)$
        \State Apply $r$ through the residual interface around $\pibase$
        and observe next state $x'$ and terminal indicator $d$
        \State Store $(x,r,x',d)$ in $\cD$
    \EndFor


    \For{gradient updates}\Comment{\textit{Replay-based actor--critic updates}}

        \State
        $\cB
        \gets
        \Call{ReplaySample}{
            \cD,
            \mathsf{Replay},
            Q_{\bar\phi},
            \pisafe_\psi
        }$

        \For{each $(x,r,x',d)\in\cB$}\Comment{\textit{Q-CBF critic target construction}}

            \If{$d=1$} $y\gets g(x)$
            \Else 
            \State Sample
                $r'\sim\pisafe_\psi(\cdot\mid x')$, compute
                $
                Q_{\bar\phi}(x',r')
                \gets
                \min_{1\leq j\leq n}
                Q_{\bar\phi_j}(x',r')
                $
                \State Set
                $
                y
                \gets
                (1-\gamma_k)g(x)
                +
                \gamma_k
                \min\left\{
                    g(x),
                    Q_{\bar\phi}(x',r')
                \right\}
                $
            \EndIf

        \EndFor

        \For{$j=1,\ldots,n$}\Comment{\textit{Critic update}}
            \State Update $\phi_j$ by descending
            $
            \widehat{\cL}_Q(\phi_j)
            =
            |\cB|^{-1}
            \sum_{(x,r,\cdot)\in\cB}
            \left(
                Q_{\phi_j}(x,r)-y
            \right)^2
            $
        \EndFor

        \State Define
        $
        Q_\phi(x,r)
        \gets
        \min_{1\leq j\leq n}
        Q_{\phi_j}(x,r)
        $

        \State Update $\psi$ via
        $
        \widehat{\cL}_{\mathrm s}(\psi)
        =
        -|\cB|^{-1}
        \sum_{x\in\cB}
        \E_{r\sim\pisafe_\psi(\cdot\mid x)}
        \left[
            Q_\phi(x,r)
            -
            \entropyreg\log\pisafe_\psi(r\mid x)
        \right]
        $ \Comment{\textit{Safety-policy update}}

        \For{$j=1,\ldots,n$} \Comment{\textit{Target-critic update}}
            \State
            $\bar\phi_j
            \gets
            (1-\tau)\bar\phi_j+\tau\phi_j$
        \EndFor

    \EndFor

\EndFor

\State Freeze $Q_\phi$ and $\pisafe_\psi$
\State \Return $Q_\phi,\pisafe_\psi$
\end{algorithmic}
\end{algorithm}


\begin{algorithm}[h]
\small
\caption{Boundary-Guided Replay Sampling: \textsc{ReplaySample}}
\label{alg:appendix-boundary-replay}
\begin{algorithmic}[1]
\Require Replay buffer $\cD$, replay configuration $\mathsf{Replay}$, target critic $Q_{\bar\phi}$, safety policy $\pisafe_\psi$

\State Read the replay mode $m$, minibatch size $B$, and mode-specific parameters from $\mathsf{Replay}$

\If{$m=\mathrm{uniform}$}
    \State Sample $B$ transitions uniformly from $\cD$ to form $\cB$
    \State \Return $\cB$
\EndIf

\If{$m=\mathrm{margin}$}
    \State Set the replay pool $\mathcal P\gets\cD$
    \For{each transition $(x,r,x',d)\in\mathcal P$}
        \State Assign weight $w(x)\gets\exp(-|g(x)|/\tau_{\mathrm{rep}})+p_{\min}$
    \EndFor
\Else
    \State Draw a uniform candidate pool $\mathcal P$ of transitions from $\cD$ using the configured pool size
    \For{each transition $(x,r,x',d)\in\mathcal P$}
        \State Sample $r_j\sim\pisafe_\psi(\cdot\mid x)$ for $j=1,\ldots,K$
        \State Estimate $\widehat V_{\bar\phi}(x)\gets\max_{1\leq j\leq K}Q_{\bar\phi}(x,r_j)$
        \State Assign weight $w(x)\gets\exp(-\beta|\widehat V_{\bar\phi}(x)|)$
    \EndFor
\EndIf

\State Normalize the weights over the transitions in $\mathcal P$
\State Sample $B$ transitions from $\mathcal P$ according to the normalized weights to form $\cB$
\State \Return $\cB$
\end{algorithmic}
\end{algorithm}

\begin{algorithm}[t!]
\small
\caption{Stage-II PPO with Counterfactual Q-CBF Distillation}
\label{alg:appendix-stage-two}
\begin{algorithmic}[1]
\Require
frozen Q-CBF $Q_\phi$,
frozen safety policy $\pisafe_\psi$,
task policy $\pitask_\theta$,
PPO value function $V_\nu$,
AMP discriminator $D_\omega$,
teacher margin $\Delta$,
Q-CBF weight $w_{\rm Q}$,
AMP weight $\lambda_s$

\For{PPO iterations}

    \State Initialize on-policy rollout buffer $\cB_{\rm PPO}$

    \For{rollout steps $t=0,\ldots,T-1$}

        \State Observe $x_t$

        \State Sample task-policy residual
        $
        r_t
        \sim
        \pitask_\theta(\cdot\mid x_t)
        $

        \State Evaluate the task proposal
        $
        c_t
        \gets
        \classkl(V_\phi(x_t))
        +
        \Delta
        -
        Q_\phi(x_t,r_t)
        $

        \If{$c_t\leq0$}
            \State Set teacher reference
            $
            \bar r_t\gets r_t
            $
        \Else
            \State Compute safety-policy residual
            $
            r_t^{\rm s}
            \sim
            \pisafe_\psi(\cdot\mid x_t)
            $

            \State Construct ordered candidate corrections
            $
            \tilde r_t^{(j)}
            =
            r_t
            +
            \lambda_j
            \bigl(
                r_t^{\rm s}-r_t
            \bigr),
            \qquad
            0<\lambda_1<\cdots<\lambda_J\leq1
            $

            \State Find the smallest $j$ satisfying
            $
            Q_\phi(x_t,\tilde r_t^{(j)})
            \geq
            \classkl(V_\phi(x_t))
            +
            \Delta
            $

            \If{such a candidate exists}
                \State
                $
                \bar r_t
                \gets
                \tilde r_t^{(j)}
                $
            \Else
                \State
                $
                \bar r_t
                \gets
                r_t^{\rm s}
                $
                as a best-effort teacher reference
            \EndIf
        \EndIf

        \State Compute counterfactual Q-CBF feedback
        $
        J_{{\rm Q},t}
        \gets
        \max\left\{
        -C,\,
        -[c_t]_+
        +
        \exp\left(
        -
        \|r_t-\bar r_t\|_W^2
        /\sigma_r^2
        \right)-1
        \right\}
        $

        \State Execute the \emph{original task-policy residual} $r_t$

        \State Observe $x_{t+1}$ and task reward terms

        \State Evaluate AMP reward $J_{{\rm AMP},t}$

        \State Form total reward
        $
        J_t
        \gets
        J_{{\rm task},t}
        +
        \delta_t
        \left(
        w_{\rm Q}J_{{\rm Q},t}
        +
        \lambda_sJ_{{\rm AMP},t}
        \right)
        $

        \State Store
        $
        (x_t,r_t,J_t,x_{t+1},
        \log\pitask_\theta(r_t\mid x_t),
        V_\nu(x_t))
        $
        in $\cB_{\rm PPO}$

    \EndFor

    \State Compute returns and generalized-advantage estimates
    $
    \{\widehat G_t,\widehat A_t\}_{t=0}^{T-1}
    $
    from $\cB_{\rm PPO}$

    \For{PPO optimization epochs}
        \For{minibatches $\mathcal{M}\subset\cB_{\rm PPO}$}

            \State Compute probability ratios
            $
            \rho_t(\theta)
            =
            \pitask_\theta(r_t\mid x_t)
            /
            \pitask_{\theta_{\rm old}}(r_t\mid x_t)
            $

            \State Update $\theta$ using the clipped PPO objective
            \[
            \frac{1}{|\mathcal{M}|}
            \sum_{t\in\mathcal{M}}
            \min\left\{
            \rho_t(\theta)\widehat A_t,\,
            \operatorname{clip}\bigl(\rho_t(\theta),1-\epsilon_{\rm PPO},1+\epsilon_{\rm PPO}\bigr)\widehat A_t
            \right\}.
            \]

            \State Update $V_\nu$ toward the rollout returns
            $\widehat G_t$

        \EndFor
    \EndFor

    \State Update AMP discriminator $D_\omega$ using
    rollout samples and reference-motion samples

    \State Keep $Q_\phi$ and $\pisafe_\psi$ fixed

\EndFor

\State \Return $\pitask_\theta$
\end{algorithmic}
\end{algorithm}
\end{document}